\documentclass[10pt]{article} 
\usepackage[preprint]{rlj}         
\usepackage{amssymb}          
\usepackage{mathtools}         
\usepackage{mathrsfs}         
\usepackage{graphicx}          
\usepackage{subcaption}        
\usepackage[space]{grffile}     
\usepackage{url}                
\usepackage{lipsum}            
\usepackage{tikz}

\usepackage{tcolorbox}
\tcbuselibrary{breakable}

\usepackage{amsthm}
\usepackage{algorithm}
\usepackage{algorithmic}
\usepackage{thmtools}
\usepackage{thm-restate}
\usepackage{bbm}

\theoremstyle{plain}
\newtheorem{theorem}{Theorem}[section]
\newtheorem{proposition}[theorem]{Proposition}
\newtheorem{lemma}[theorem]{Lemma}
\newtheorem{corollary}[theorem]{Corollary}
\theoremstyle{definition}

\theoremstyle{remark}

\tcbset{
  noparskip/.style={
    colback=gray!20,
    colframe=white,
    boxrule=0pt,
    sharp corners,
    boxsep=0pt,
    left=0pt,
    right=0pt,
    top=0pt,
    bottom=0pt, 
    breakable}
}

\title{Offline Action-Free Learning of Ex-BMDPs\\ by Comparing Diverse Datasets}

\setrunningtitle{Offline Action-Free Learning of Ex-BMDPs by Comparing Diverse Datasets}

\author{Alexander Levine\textsuperscript{1}, Peter Stone\textsuperscript{1,2}, Amy Zhang\textsuperscript{1}}

\emails{alevine0@cs.utexas.edu, \ \
pstone@cs.utexas.edu, \ \ amy.zhang@austin.utexas.edu}

\affiliations{
$^{1}$\textbf{The University of Texas at Austin}\,\,
$^{2}$\textbf{Sony AI}\\
}

\keywords{ representation learning, action-free RL, Ex-BMDP,  controllable state representations} 

\begin{document}

\maketitle  
\begin{abstract}
While sequential decision-making environments often involve high-dimensional observations, not all features of these observations are relevant for control. In particular, the observation space may capture factors of the environment which are not controllable by the agent, but which add complexity to the observation space. The need to ignore these ``noise'' features in order to operate in a tractably-small state space poses a challenge for efficient policy learning. Due to the abundance of video data available in many such environments, task-independent representation learning from action-free offline data offers an attractive solution. However, recent work has highlighted theoretical limitations in action-free learning under the Exogenous Block MDP (Ex-BMDP) model, where temporally-correlated noise features are present in the observations. To address these limitations, we identify a realistic setting where representation learning in Ex-BMDPs becomes tractable: when action-free video data from multiple agents with differing policies are available. Concretely, this paper introduces CRAFT (Comparison-based Representations from Action-Free Trajectories), a sample-efficient algorithm leveraging differences in controllable feature dynamics across agents to learn representations. We provide theoretical guarantees for CRAFT's performance and demonstrate its feasibility on a toy example, offering a foundation for practical methods in similar settings.
\end{abstract}

\section{Introduction}
Many sequential decision-making settings, such as robotic navigation environments, involve high-dimensional observations with many uncontrollable noise features. In order to efficiently learn policies for many downstream tasks, techniques for task-independent representation learning have been proposed. These techniques learn encoders that map the large observation space into a much smaller set of learned latent states, which can then be used to learn policies for downstream objectives more efficiently than learning from observations directly. 

In some such settings, such as social navigation, large amounts of offline \textit{video} data are available, collected either with similar robots or with human agents. In video data, observations are available, but \textit{action} labels are not. Past work has shown that this offline data can be used to learn encodings that can be leveraged for downstream tasks \citep{ma2023vip,nair2023r3m,seo2022reinforcement}.

However, in recent work, \citet{misra2024towards} has shown an important theoretical limitation to this approach: for some important classes of environments, efficient action-free representation learning is not possible. In particular, the Exogenous Block MDP (Ex-BMDP) model \citep{efroni2022provably} describes a class of environments where observations depend both on a deterministic, action-controlled latent state, and a potentially high-dimensional temporally-correlated noise factor, which is action-independent.
The goal of representation learning in the Ex-BMDP setting is to learn a mapping from the observation space to the much-smaller space of action-controlled latent states, while ignoring the noise factor.
\citet{misra2024towards} show that, even with high coverage over latent states, representation learning from  Ex-BMDPs is not possible in general. This property of Ex-BMDPs is in contrast to \textit{Block} MDPs, where the observation noise is not time-correlated, and which \citet{misra2024towards} demonstrates \textit{are} amenable to efficient action-free representation learning. At a high level, \citet{misra2024towards}'s hardness result stems from the fact that, without action labels, action-controllable features are indistinguishable from uncontrollable features. (For example, if the observations capture both the controllable ego-agent's state and other uncontrollable ``background'' agents' states, it is ambiguous what state should be encoded in the learned representation.)

In this work, we describe a realistic setting where representation learning for Ex-BMDPs is in fact tractable, and propose a provably sample-efficient algorithm for this setting. Specifically, we consider cases where videos are available of  \textit{multiple distinct agents} operating in the same environment. Intuitively, the idea is that \textit{controllable} latent features will \textit{differ} in their dynamics between datasets collected by different agents, while \textit{uncontrollable} features will have \textit{the same} dynamics in the two datasets. Our main result is that, if two agents' policies sufficiently differ at every latent state, then, under assumptions similar to those in \citet{misra2024towards} and \citet{efroni2022provably},  sample-efficient representation learning from offline action-free data collected by the two agents is possible. 

To show this fact, we propose a provably sample-efficient algorithm for Ex-BMDP representation learning without action labels, which we call \textbf{C}omparison-based \textbf{R}epresentations from \textbf{A}ction-\textbf{F}ree \textbf{T}rajectories, or \textbf{CRAFT}. CRAFT enjoys a sample complexity that depends only on the size of, and coverage assumptions on, the controllable latent states of the environment, and, logarithmically, on the size of the encoder hypothesis class. The sample complexity has no explicit dependence of the size of the space of exogenous noise. 
At a high level, CRAFT works by clustering sequential observation-pairs together based on how likely the pairs are to have been observed by each agent. 

In this work, we introduce CRAFT, prove its correctness and sample-complexity, and validate its use on a toy example problem. To our knowledge, this is the first work to propose a provably sample-efficient algorithm for action-free offline representation learning in the Ex-BMDP setting. While this work is theoretical in nature due to some restrictive assumptions on the setting (which are inherited from the prior work upon which we build; see discussion in Section \ref{sec:discussion}), we expect that the CRAFT algorithm can inspire practical methods that rely on the same principle of comparing action-free video datasets from diverse agents, in order to extract controllable feature representations.
\section{Background}
In this section, we define notation, formally introduce the action-free offline Ex-BMDP setting, and state our technical assumptions. 
\subsection{General Notation}
We use $[N]$ to denote the set $\{1,...,N\}$. For a sequence $x_1,..., x_N$, we use $x_{i:j}$ to denote the subsequence $x_i,x_{i+1}...,x_j$. For multisets $A$, $B$, we use $A \uplus B$ to denote the union of the two multisets, with multiplicities added.
\subsection{Ex-BMDP Framework}
The Ex-BMDP model describes a class of sequential decision-making environments where an agent's actions only operate on a hidden latent state, while the observations that the agent receives are also functions of a temporally-correlated exogenous ``noise'' factor, in addition to this controllable latent state.  Following \citet{efroni2022provably}, we consider the \textit{finite horizon} variant of this model, and also assume that the controllable latent dynamics are \textit{deterministic}.\footnote{\citet{efroni2022provably} presents an algorithm for efficient online representation learning of Ex-BMDPs with \textit{near-}deterministic latent dynamics: the controllable latent state deviates from deterministic behavior with frequency $\ll$ one time per episode. See Section \ref{sec:discussion} for further discussion.} Formally, a (reward-free) Ex-BMDP can be described as a tuple, 
    $\mathcal{M} = \langle H, \mathcal{A},\mathcal{X}_{1:H},\mathcal{S}^*_{1:H},\mathcal{E}_{1:H},\mathcal{Q}_{1:H},T_{2:H},\mathcal{T}^e_{2:H},s_1^*,P^e_1\rangle,$
where $H$ is the horizon (the number of steps per episode).
At each timestep $h \in [H]$, the observation $x_h \in \mathcal{X}_h$ is determined by two latent factors, $s_h^* \in \mathcal{S}_h^*$ and $e_h \in \mathcal{E}_h$. We assume that $\mathcal{S}_h^*$ is finite, while the $\mathcal{E}_h$ and the observation space $\mathcal{X}_h$ may be continuous.

The controllable latent state $s_h^*$ evolves deterministically, depending on the action $a_h$ taken by the agent: $s_{h+1}^* = T_{h+1}(s_h^*,a_h),$ where $T_{h+1} \in \mathcal{S}^*_h \times \mathcal{A} \rightarrow \mathcal{S}_{h+1}^*$ is a deterministic function, and $\mathcal{A}$ is the set of possible actions, which we assume to be finite. Note that $s_1^*$ is a constant. (Each episode starts at the same controllable latent state, so $\mathcal{S}_1^* = \{s_1^*\}.$) 

By contrast, the exogenous (noise) state evolves stochastically as a Markov chain, independent of actions. The initial exogenous state $e_1$ is sampled from the distribution $P^e_1 \in \mathcal{P}(\mathcal{E}_1)$, and subsequent observations are sampled as $e_{h+1} \sim \mathcal{T}^e_{h+1}(e_h)$, where $\mathcal{T}^e_{h+1} \in \mathcal{E}_h \rightarrow \mathcal{P}(\mathcal{E}_{h+1})$. We can refer to the distribution of exogenous states $e_h$ at time $h$ as  $P^e_h = \mathcal{T}^e_{h}( \mathcal{T}^e_{h-1}( ...\mathcal{T}^e_{2}(P^e_1)...))$, and the \textit{joint} distribution of exogenous states $e_h$ and $e_{h+1}$ as $P^e_{h:h+1}$.\footnote{ Note that in general, $P^e_{h:h+1} \neq P^e_h \times P^e_{h+1}$. }

The observation $x_h$ is then sampled as $x_h \sim \mathcal{Q}_h(s_h^*,e_h)$, where $\mathcal{Q}_h \in \mathcal{S}^*_h \times \mathcal{E}_h \rightarrow \mathcal{P}(\mathcal{X}_h)$ is the \textit{emission function}.  Under the Ex-BMDP model, we assume that the latent variables $s_h^*$ and $e_h$ can always be inferred from $x_h$: that is, $\mathcal{Q}_h$ has a deterministic  inverses $\phi^*_h$ and $\phi^{e}_h$, such that if $x_h$ is sampled from $\mathcal{Q}_h(s_h^*,e_h)$, then $\phi^*_h(x_h) = s_h^*$ and $\phi^{e}_h(x_h) = e_h$. (This assumption is the \textit{block} assumption referred to in the name ``Exogenous Block MDP.'')

The agent does not have access to $s_h^*$, $e_h$, or the ``ground-truth'' encoders $\phi^*_h, \phi^{e}_h$; instead, it only has access to the observations $x_h$. The goal of representation learning is to learn an encoder $\phi_h: \mathcal{X}_h \rightarrow \mathbb{N}$ for each timestep $h$, that approximates $\phi^*_h$, up to label permutation. (We are \textit{not} interested in learning the exogenous encoder $\phi^{e}_h$, because it is assumed that this noise factor is irrelevant for control, and may be very large.)

\subsection{Action-Free, Offline Setting} \label{sec:action_free}
In this work, we consider a setting where the learner has access to multiple sets of  offline trajectories collected by different agents, but where only the observations $x_h$, and \textit{not} the actions $a_t$, are available. For simplicity, in this work we assume that there are only datasets from two distinct agents, but the proposed method could be straightforwardly generalized to support more agents.

We refer to the two trajectory datasets as $\tau_A$ and $\tau_B$.
Each trajectory in $\tau_A$ (or $\tau_B$) is a sequence of observations $x_{1:H}$. We use $(\tau_A)_{h:h+i}$ to refer to the multiset of tuples of observations $(x_h,x_{h+1},...,x_{h+i})$ for each trajectory in $\tau_A$, and use $\tau_A[\{i,i',i''\}]$ to refer to the subset consisting of three \textit{trajectories} (indexed $i,i'$ and $i''$) in $\tau_A$.  We use $\mathcal{D}_A^*(s^*_{h},s^*_{h+1})$ to refer to the multiset consisting of all observation pairs $(x_h,x_{h+1}) \in (\tau_A)_{h:h+1}$ such that $\phi^*_h(x_h) =s^*_{h} $ and $\phi^*_{h+1}(x_{h+1})=s^*_{h+1} $; and  $\mathcal{D}^*(s^*_{h},s^*_{h+1}) := \mathcal{D}_A^*(s^*_{h},s^*_{h+1}) \uplus \mathcal{D}_B^*(s^*_{h},s^*_{h+1})$. We also define $\mathcal{D}^*(s^*_{h})$ as the multiset of observations in $x_h \in (\tau_A)_{h} \uplus (\tau_B)_{h} $ such that $\phi^*_h(x_h) =s^*_{h}$.

\subsection{Technical Assumptions: Data Collection Method}
As in previous works in offline representation learning in the Ex-BMDP setting \citep{misra2024towards,islam2023principled,levine2024multistep,lamb2023guaranteed}, we assume that the agents' actions $a_h$ are chosen \textit{independently} of the observations $x_{1:h}$, given the controllable latent states $s^*_{1:h}$. In other words, roughly speaking, we assume that the agents used to collect the offline data choose actions based \textit{only} on the controllable latent state, not on the full observation. \citet{misra2024towards} justifies this assumption by positing that the offline data are likely collected by expert agents which ``would not make decisions based on noise.'' We discuss this assumption further in Section \ref{sec:discussion}.

Beyond sharing this noise-independence assumption, our technical assumptions on the data-collection policies are otherwise significantly weaker that those in \citet{misra2024towards}. While \citet{misra2024towards} assumes that each trajectory is generated by a \textit{Markovian} policy (i.e, that $a_h  \sim \pi(s^*_h)$, for some $\pi \in \mathcal{S}^*_h \rightarrow \mathcal{P}(\mathcal{A})$), and furthermore that the policies used to generate each trajectory are chosen i.i.d., we make neither such assumption. In other words, we allow for both non-Markovian  behavioral policies -- for example, we could have a policy in the form $a_h \sim \pi(s^*_{1:h})$ -- and non-i.i.d. sampling of behavioral policies between episodes -- for example: the data collector could evolve over time between episodes, in order to, for instance, maximize the diversity of visited latent states. 

Explicitly, our \textit{only} assumption on the data-collection mechanism is that the process which generates action sequences $a_{1:H}$ -- and therefore, equivalently, controllable latent-state sequences $s^*_{1:H}$ -- is independent from observation noise over \textit{both entire datasets}. Formally, for a dataset $\tau$ that consists of trajectories $x_{1:H}$, let $\phi^*(\tau)$ denote the corresponding set of controllable latent state trajectories $s^*_{1:H}$, and $\phi^{e}(\tau)$ denote the corresponding set of exogenous state trajectories $e_{1:H}$. Then, our only requirement on the data collection mechanism is that it ensures:
\begin{equation}
\begin{split}
&\Pr(\tau_A, \tau_B) = \Pr(\phi^*(\tau_A), \phi^*(\tau_B))\\
&\cdot \Pr_{P_1^e,\mathcal{T}^e}(\phi^{e}(\tau_A))\cdot \Pr_{P_1^e,\mathcal{T}^e}( \phi^{e}(\tau_B))
\cdot  \Pr_Q(\tau_A|\phi^*(\tau_A),\phi^{e}(\tau_A) 
) \cdot \Pr_Q(\tau_B|\phi^*(\tau_B),\phi^{e}(\tau_B) 
). \label{eq:noise_free_policy_property}
\end{split}
\end{equation}
To see why Equation \ref{eq:noise_free_policy_property} is a sufficiently strong assumption to allow for useful analysis despite its apparent generality, fix any two latent states $s^*_{h},s^*_{h+1} \in \mathcal{S}_{h:h+1}^*$, and  consider the multiset $\mathcal{D}_A^*(s^*_{h},s^*_{h+1})$ as defined in Section \ref{sec:action_free}; also let $n := |\mathcal{D}_A^*(s^*_{h},s^*_{h+1})|$. Then the marginal distribution of $\mathcal{D}_A^*(s^*_{h},s^*_{h+1})$ can be described as:
\begin{equation}
        \mathcal{D}_A^*(s^*_{h},s^*_{h+1}) \sim [(\mathcal{Q}(s_h^*,e_h), \mathcal{Q}(s_{h+1}^*,e_{h+1})) | e_h, e_{h+1} \sim P^e_{h:h+1} ]^n.
\end{equation}
We see that $  \mathcal{D}_A^*(s^*_{h},s^*_{h+1})$ consists of i.i.d. samples from a fixed, policy-independent distribution. Consequently, this property will frequently allow us to use standard concentration bounds in our analysis, while still allowing for a wide class of non-Markovian, non-i.i.d. behavioral policies.

While we do not require the behavioral policies to be Markovian, it will be useful to refer to the ``empirical policies'' $\pi^{emp.}_A$ and  $\pi^{emp.}_B$, defined as:
\begin{equation}
    \pi^{emp.}_A(s^*_{h+1} | s^*_{h}) := \frac{|\mathcal{D}_A^*(s^*_{h},s^*_{h+1})|}{\sum_{s'\in \mathcal{S}^*_{h+1}} |\mathcal{D}_A^*(s^*_{h},s')|},
\end{equation}
and likewise for $\pi^{emp.}_B$. This is the empirical likelihood in the provided data that agent A (respectively, B) chooses an action that results in a transition from $s^*_{h}$ to $s^*_{h+1}$.

\subsection{Technical Assumptions: Coverage, Policy Diversity, and Realizability} \label{sec:cpdr}
In order to learn accurate latent state encoders $\phi_{1:H}$, we need to ensure adequate coverage over all latent states $s_h$.  For all timesteps $h$, and all pairs of latent states $(s^*_{h},s^*_{h+1}) \in \mathcal{S}_h^* \times \mathcal{S}_{h+1}^*$ such that $s^*_{h+1} = T_h(s^*_h,a)$ for some action $a$,  we require that 
    \begin{equation}
       \frac{ | \mathcal{D}^*(s^*_{h}, s^*_{h+1})|}{|\tau_A|+|\tau_B|} \geq \nu, \label{eq:pair_coverage}
    \end{equation}

for some known lower-bound $\nu$. This coverage assumption is presented in terms of the \textit{actually realized offline datasets} $\tau_A$ and $\tau_B$. By contrast, \cite{misra2024towards} assumes that trajectories in the offline dataset are sampled i.i.d., and makes coverage assumptions on the \textit{policies} used sample them.

We also require that the two agents, which produced datasets $\tau_A$ and $\tau_B$, behaved sufficiently differently so that we can infer the latent dynamics from their differences. In particular, for some known lower bound $\alpha > 0$, we require that, $\forall h \in [H-1], \forall s^*_h \in \mathcal{S}^*_h$, and for any two successor states $s_{h+1}^*, s_{h+1}'^* \in \mathcal{S}^*_{h+1}$, such that $s^*_h$ can transition to either $s_{h+1}^*$  or $s_{h+1}'^*$ under $T_h$, we have, either:

 \begin{equation}
       e^\alpha \cdot \frac{  \pi^{emp.}_B(s'^*_{h+1}| s^*_{h})}{\pi^{emp.}_B(s^*_{h+1}| s^*_{h})} \leq \frac{ \pi^{emp.}_A(s'^*_{h+1}| s^*_{h})}{\pi^{emp.}_A(s^*_{h+1}| s^*_{h}) }, \text{ or,  }   e^\alpha \cdot \frac{  \pi^{emp.}_A(s'^*_{h+1}| s^*_{h})}{\pi^{emp.}_A(s^*_{h+1}| s^*_{h})} \leq \frac{ \pi^{emp.}_B(s'^*_{h+1}| s^*_{h})}{\pi^{emp.}_B(s^*_{h+1}| s^*_{h}) }. \label{eq:alpha_seperation}
    \end{equation}

In other words, the relative likelihood of transitioning to $s'^*_{h+1}$, versus transitioning to $s^*_{h+1}$, is different in $\tau_A$ and $\tau_B$ by a multiplicative factor of at least $e^\alpha$.

Finally, we also require that the difference in $\textit{total}$ state coverage between $\tau_A$ and $\tau_B$ for any pair of sequential latent states $(s^*_h,s^*_{h+1})$ is not \textit{too} extreme. We require that, for a known lower-bound $\eta$:
\begin{equation}
    \frac{|\mathcal{D}_A^*(s^*_{h},s^*_{h+1})|}{|\mathcal{D}^*(s^*_{h},s^*_{h+1})|} \geq \eta, \label{eq:eta_coverage}
\end{equation}
and likewise for $\mathcal{D}_B^*(s^*_{h},s^*_{h+1})$. 
Our sample-complexity bound also depends on an additional parameter $\nu'$, which does \textit{not} need to be known a priori.  This is the minimum single-state coverage~ratio:

    \begin{equation}
      \nu' := \min_{h\in [H],s_h^*\in \mathcal{S}_h^*} \frac{ | \mathcal{D}^*(s^*_{h})|}{|\tau_A|+|\tau_B|} .
    \end{equation}
\textbf{Function Approximation Assumptions.} We assume access to hypothesis classes of encoder functions $\Phi_{1:H}$, as well as binary classification functions $\mathcal{G}_h \subseteq \mathcal{X}_h \rightarrow \{0,1\}$. We make standard~realizability assumptions (in brief, $\phi^*_h \in \Phi_h$, and $\forall s^*_h,s'^*_h, \exists g \in  \mathcal{G}_h$ such that $g$ can perfectly distinguish between observations of $s^*_h$ and those of $s'^*_h$ ) and assume access to training oracles. See Appendix \ref{sec:aaprealizability} for further information about these assumptions.
We use $|\Phi|$ to denote $\max_h |\Phi_h|$. We also assume a known upper-bound $N_s$ on $\max_h |\mathcal{S}_h|$; that is, the maximum output range of any encoder in $\Phi_h$.  

\section{Method}
In this section, we describe the CRAFT algorithm. First, however, we motivate its design by examining a simpler version of the problem setting and of the algorithm.
\subsection{Intuition: Single-step, Binary Action Case} \label{sec:draft_two_step}
In this section, we present a toy algorithm for an extremely simplified version of the  Ex-BMDP representation learning problem: an explanation of the toy algorithm captures some of the intuitions of CRAFT, while the limitations of the toy algorithm will motivate some of the less-intuitive algorithmic details. We can call this naive, ``first draft'' form of the CRAFT algorithm ``DRAFT.''

We first consider the Ex-BMDP model with $H = 2$ and $|\mathcal{A}| = |\mathcal{S}^*_2| = 2$. In this setting, the representation learning problem reduces to the task of learning to distinguish the two latent states $s_2^*$ and $s_2'^* \in \mathcal{S}^*_2$ that can occur at  $h = 2$. (See Figure \ref{fig:2_state_diagram}.)
\begin{figure}[h!]
    \centering

\begin{subfigure}[t]{0.49\textwidth}
\centering
    {\center
\begin{tikzpicture}[>=stealth, ->, node distance=2cm]
    \tikzstyle{state} = [circle, draw, minimum size=0.8cm, inner sep=0pt]
    \node[state] (s1) at (0,0) {$s_1^*$};
    \node[state] (s2) at (3,.6) {$s_2^*$};
    \node[state] (s2p) at (3,-0.6) {$s_2'^*$};
    \draw (s1) -- (s2);
    \draw (s1) -- (s2p);

\end{tikzpicture}

}
\caption{Latent-state transition graph.}
    \label{fig:2_state_diagram}
\end{subfigure}
\begin{subfigure}[t]{0.49\textwidth}
    \centering
\includegraphics[width=\linewidth]{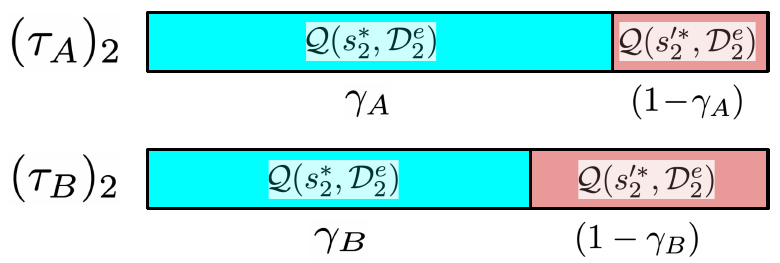}
    \caption{Composition of the datasets $(\tau_A)_2$ and $(\tau_B)_2$.}
    \label{fig:draft_two_step}
\end{subfigure}
\caption{Dynamics and composition of the two-step Ex-BMDP example in Section \ref{sec:draft_two_step}.}
\end{figure}
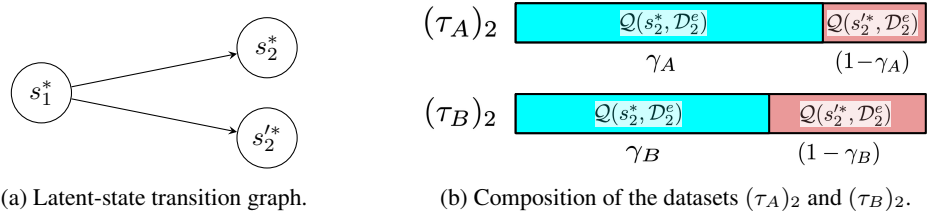

We will also assume that $|\tau_A| = |\tau_B|  =m$. Then, $(\tau_A)_1$ and $(\tau_B)_1$ both consist entirely of $m$ i.i.d. samples of the same distribution $\mathcal{Q}(s_1^*,P^e_1 )$. The structures of $(\tau_A)_2$ and $(\tau_B)_2$ are (slightly) more complex. If we let $\gamma_A := \pi^{emp.}_A(s^*_2|s^*_1)$ and $\gamma_B := \pi^{emp.}_B(s^*_2|s^*_1)$, then we see that the dataset  $(\tau_A)_2$ consists of $m\cdot\gamma_A$ i.i.d. samples of the distribution $\mathcal{Q}(s_2^*,P^e_2)$ and $m\cdot(1-\gamma_A)$ i.i.d. samples of the distribution $\mathcal{Q}(s_2'^*,P^e_2)$, while $(\tau_B)_2$ consists of $m\cdot\gamma_B$ i.i.d. samples of the distribution $\mathcal{Q}(s_2^*,P^e_2)$ and $m\cdot(1-\gamma_B)$ i.i.d. samples of the distribution $\mathcal{Q}(s_2'^*,P^e_2)$. (See Figure \ref{fig:draft_two_step}.)

Now, by assumption, the agents generating datasets $\tau_A$ and $\tau_B$ are not behaviorally identical: this means that $\gamma_A \neq \gamma_B$. Without loss of generality, assume that $\gamma_A > \gamma_B$. Our key insight is that, in the limit as $m \rightarrow \infty$, the Bayes optimal classifier to distinguish a sample $x_2$ selected from $(\tau_A)_2$ from a sample $x_2$ selected from $(\tau_B)_2$ is in fact (up to label permutation) the latent state encoder $\phi_2^*(x_2)$. Concretely, consider a classifier $\phi_2'$ trained to minimize the 0-1 classification loss between $(\tau_A)_2$ and $(\tau_B)_2$. In the limit of infinite data, we can define this classification loss function as:
\begin{equation}
    \mathcal{L}_{pop}(\phi_2) := \lim_{m \rightarrow\infty}\mathop{\mathbb{E}}_{x \in (\tau_A)_2}\hspace{-.3em} \phi_2(x) +  \mathop{\mathbb{E}}_{x \in (\tau_B)_2} \hspace{-.3em}1-\phi_2(x),  \label{eq:draft_inf}
\end{equation}

From Equation \ref{eq:draft_inf} and the composition of the datasets:
\begin{equation}
\begin{split}
        \mathcal{L}_{pop}(\phi_2) &= \gamma_A \mathop{\mathbb{E}}_{x \sim \mathcal{Q}(s_2^*,P^e_2)} \phi_2(x) + (1-\gamma_A)\mathop{\mathbb{E}}_{x \sim \mathcal{Q}(s_2'^*,P^e_2)}  \phi_2(x) \\ &+\gamma_B\mathop{\mathbb{E}}_{x \sim \mathcal{Q}(s_2^*,P^e_2)}(1- \phi_2(x)) + (1-\gamma_B)\mathop{\mathbb{E}}_{x \sim \mathcal{Q}(s_2'^*,P^e_2)} (1-\phi_2(x))\\
    &= (\gamma_A- \gamma_B)[\mathop{\mathbb{E}}_{x \sim \mathcal{Q} (s_2^*,P^e_2)}\phi_2(x) + \mathop{\mathbb{E}}_{x \sim \mathcal{Q} (s_2'^*,P^e_2)}1- \phi_2(x) ]  + C\\
    &= - (\gamma_A- \gamma_B)(\Pr(\phi_2(x)= 0|\phi^*(x) = s_2^*) + \Pr(\phi_2(x)= 1|\phi^*(x) = s_2'^*)) + C
\end{split} \label{eq:draft_asymptotic_loss}
\end{equation}
where $C$ is independent of $\phi_2$. Under the mapping $(0 \rightarrow s^*_2, 1\ \rightarrow s_2'^*),$  we see that $\mathcal{L}_{pop}(\phi_2)$ scales linearly with the rate that $\phi_2$ produces incorrect encodings, with a $\delta$-increase in $\mathcal{L}_{pop}$ corresponding to an $\mathcal{O}((\gamma_A-\gamma_b)\cdot \delta)$ increase in encoder failure. In particular, $\mathcal{L}_{pop}$ is uniquely minimized by the ground-truth encoder $\phi^*_2$.
We now examine how this simple algorithm functions with finite datasets:
\begin{algorithm}[h!]
    \begin{algorithmic}
    \REQUIRE{Datasets $\tau_A$, $\tau_B$ with $H=2$, hypothesis class $\Phi_2 \in \mathcal{X}_2 \rightarrow \{0,1\}$.}
    \STATE{Let $\phi_1' := \mathcal{X}_1 \rightarrow 0$, and $\phi'_2 := \arg \min_{\phi_2 \in \Phi_2} \mathcal{L}(\phi_2)$, where:
    \vspace{-9pt}
    \begin{equation}
        \mathcal{L}(\phi_2) := \frac{1}{|\tau_A|}\hspace{-3pt}\sum_{x\in(\tau_A)_2} \hspace{-0.6em}\phi_2(x) + \frac{1}{|\tau_B|}\hspace{-3pt}\sum_{x\in(\tau_B)_2}\hspace{-0.6em}1-\phi_2(x).
    \end{equation}}
    \vspace{-9pt}
\STATE{\hspace{-1em}\textbf{Return: } $\phi'_1,\phi'_2$}
    \end{algorithmic}
    \caption{``DRAFT'' algorithm for $H=2$ Ex-BMDPs.} \label{alg:draft}
\end{algorithm}

With finite $m$, our main concern is overfitting: if $\Phi_2$ is large enough such that some  $\phi_2' \in \Phi_2$ can perfectly distinguish the samples that happen to fall into $(\tau_A)_2$ from those in $(\tau_B)_2$, then this $\phi_2'$ will attain a lower empirical loss than $\phi^*_2$, while being bad at distinguishing  $\mathcal{Q}(s_2^*,P^e_2)$ from $\mathcal{Q}(s_2'^*,P^e_2)$ in general. However, as long as $|\Phi_2|$  is controlled, we can use standard concentration inequalities to limit this overfitting. In particular, 
\begin{equation}
    | \mathcal{L}(\phi_2)-\mathcal{L}_{pop}(\phi_2)| \approx \mathcal{O}(1/\sqrt{m}). \label{eq:draft_concent}
\end{equation}
By combining Equations \ref{eq:draft_concent} and \ref{eq:draft_asymptotic_loss}, we can determine how quickly $\phi'_2 =\arg \min \mathcal{L}(\cdot)$ will approach  $\phi^*_2 = \arg \min \mathcal{L}_{pop}(\cdot)$ as $m$ increases, in terms of \textit{accuracy as a latent state encoder}. To ensure $\phi_2'$ approximates $\phi_2^*$ with a failure rate of at most $\epsilon$, we need $m \approx \mathcal{O}\left(\frac{1}{(\gamma_A-\gamma_b)^2\epsilon^2} \right) \text{ samples.}$
Intuitively, the smaller the difference between behavior policies of the two agents $(\gamma_A - \gamma_B)$, the more samples are required to attain a given accuracy of encoder.
\subsubsection{``DRAFT'' doesn't generalize easily to the long-horizon setting}
A naive first attempt to extend ``DRAFT'' to the $H>2$ case might be to apply it \textit{recursively}. That is, once the two distinct latent states at $h=2$ can be decoded, we can extract from $(\tau)_A$ and $(\tau)_B$ the trajectories which contain (say) $s^*_2$, and then run DRAFT again on these samples, to obtain an encoder that can separate the two states into which $s^*_2$ can transition.\footnote{Here, we are continuing to assume $|\mathcal{A}| =2$, and that the two actions have different effects from each other in each state.} We can then repeat this procedure for $s'^*_2$. If $s'^*_2$ and $s^*_2$ both transition to the same latent state (say $s'^*_3$), we can easily detect this situation by attempting to learn binary classifiers between the observations of each state that succeeds  $s^*_2$ and each state that succeeds  $s'^*_2$: if it is impossible distinguish these observations better than by random chance, then the two successor states are the same:
\begin{figure}[h!]
    \centering
\includegraphics[width=\linewidth]{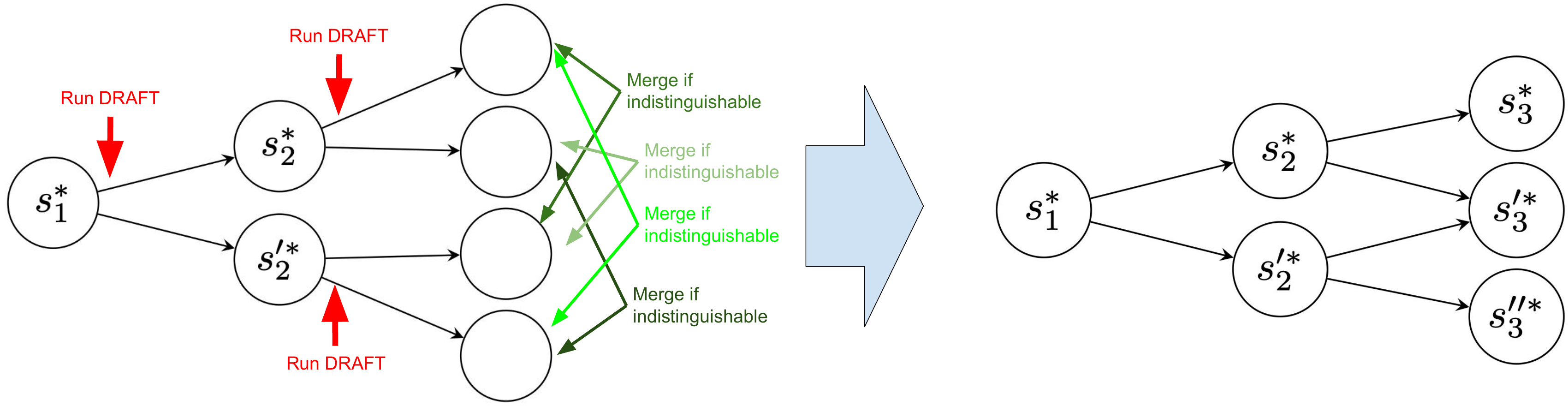}
\caption{Illustration of recursive use of the ``DRAFT'' algorithm.}
\label{fig:recursive_draft}
\end{figure}

However, it turns out that it is difficult (and may be impossible) to prove an efficient sample-complexity bound for this recursive algorithm. This is for two reasons:
\begin{enumerate}
    \item After the first timestep, the input datasets to subsequent applications of DRAFT are corrupted by mis-classified samples, such that the datasets are no longer mixtures of i.i.d. samples from distributions $\mathcal{Q}(s^*,P^e_h)$ for multiple values of $s^*\in\mathcal{S}_h^*$.
    \item DRAFT is highly sensitive to small changes in its input dataset.
\end{enumerate}
To see (1), note that the encoder $\phi_2'$ returned by the first application of DRAFT will misclassify some $\mathcal{O}(\sqrt{m}/(\gamma_A-\gamma_B))$ samples. Moreover, these misclassified samples will \textit{not} be chosen uniformly: 
the encoder $\phi'_2$ may rely on some spurious features of the observations $x_2$, which depend on $e_2$, to classify these observations. Consequently, because $e_3$ \textit{also} depends on $e_2$, the exogenous noise distributions of the observations $x_3$ of state $s'^*_3$  (for instance, in the dynamics example in Figure~\ref{fig:recursive_draft}) present in the datasets for the recursive DRAFT instances associated with $s^*_2$ and $s'^*_2$  will differ from each other, and each will differ from $\mathcal{Q}(s'^*_3,P^e_h)$, in a way that depends on the choice of $\phi'_2$.
Moreover, because $\phi'_2$ is trained to distinguish $\tau_A$ from $\tau_B$, this distribution shift may have different effects on the distributions of observations from the two agents.

For (2), consider the (single step) DRAFT algorithm with some small number $\epsilon_{bad}\cdot m$ of samples from $\mathcal{X}_2$ arbitrarily introduced to the datasets $(\tau_A)_2$ and $(\tau_B)_2$. For concreteness, we will replace some of the samples in $(\tau_A)_2$ for which $\phi_2^*(x) = s'^*_2$ with ``bad'' samples for which it \textit{still} holds that $\phi_2^*(x'_{bad}) = s'^*_2$, but which are not drawn i.i.d. from $\mathcal{Q}(s'^*_2,P^e_2)$. Instead, we assume that these samples belong to some part of the support of $\mathcal{Q}(s'^*_2,P^e_2)$ which is typically sampled with negligible probability; we can call their distribution $\mathcal{D}'_{bad}$. Similarly, we replace $\epsilon_{bad}\cdot m$ samples of $(\tau_B)_2$ for which $\phi_2^*(x) = s^*_2$ with ``bad'' samples for which $\phi_2^*(x_{bad}) = s^*_2$, but which are drawn from $\mathcal{D}_{bad}$. 
We consider the infinite-dataset limit. From Equation \ref{eq:draft_inf} and the composition of the datasets:
\begin{equation*}
\begin{split}
        &\mathcal{L}_{pop}(\phi_2) = \gamma_A \hspace{-.8em}\mathop{\mathbb{E}}_{x \sim \mathcal{Q}(s_2^*,P^e_2)}\hspace{-.8em} \phi_2(x)+  (1-\gamma_A - \epsilon_{bad})\hspace{-.6em}\mathop{\mathbb{E}}_{x \sim \mathcal{Q}(s_2'^*,P^e_2)} \hspace{-.5em} \phi_2(x)  + \epsilon_{bad}\hspace{-.8em}\mathop{\mathbb{E}}_{x \sim \mathcal{D}_{bad}'}\hspace{-.7em}(\phi_2(x)) 
        \\ & +(1-\gamma_B)\hspace{-.8em}\mathop{\mathbb{E}}_{x \sim \mathcal{Q}(s_2'^*,P^e_2)}\hspace{-.8em} (1-\phi_2(x)) +(\gamma_B-\epsilon_{bad})\hspace{-.8em}\mathop{\mathbb{E}}_{x \sim \mathcal{Q}(s_2^*,P^e_2)}\hspace{-.8em}(1- \phi_2(x)) +\epsilon_{bad}\hspace{-.8em}\mathop{\mathbb{E}}_{x \sim \mathcal{D}_{bad}}\hspace{-.7em}(1- \phi_2(x)) 
        \\
    &= - (\gamma_A- \gamma_B+\epsilon_{bad})\big(\Pr(\phi_2(x)= 0|x \sim \mathcal{Q}(s_2^*,P^e_2))  + \Pr(\phi_2(x)= 1|x \sim \mathcal{Q}(s_2'^*,P^e_2))\big) \\
    &\hspace{6.5em}-\epsilon_{bad}\big(\Pr(\phi_2(x)= 1|x \sim \mathcal{D}_{bad})  + \Pr(\phi_2(x)= 0|x \sim  \mathcal{D}'_{bad})\big) +C
\end{split}
\end{equation*}

Note that the ground-truth encoder $\phi^*_2$ has loss $\mathcal{L}_{pop}(\phi^*_2) = -(\gamma_A-\gamma_B+\epsilon_{bad}) + C$. However, we can construct an encoder $\phi_2'$ that incorrectly encodes all samples in $\mathcal{D}_{bad}$ as belonging to $s_2'^*$ (i.e., returns 1 on these samples), and incorrectly encodes all samples in $\mathcal{D}'_{bad}$ as belonging to $s_2^*$. Furthermore, we can construct this $\phi_2'$ to also have an accuracy of only $1- \epsilon_{bad}/(\gamma_A-\gamma_B + \epsilon_{bad})$ on the samples in the ``natural'' distributions $\mathcal{Q}(s_2^*,P^e_2)$ and $\mathcal{Q}(s_2'^*,P^e_2)$. Surprisingly, by the above expression for $\mathcal{L}_{pop}$, we see that this less-accurate encoder \textit{also} has loss $\mathcal{L}_{pop}(\phi'_2) = -(\gamma_A-\gamma_B+\epsilon_{bad}) + C$. Therefore, if this encoder $\phi'_2$ is included in the hypothesis class $\Phi_2$, then the ERM step of Algorithm \ref{alg:draft} may just as easily return   $\phi'_2$ rather than $\phi^*_2$. (Furthermore, the realizability assumption does not guarantee that any lower-loss ``third option'' encoders exist.)  Consequently,  the misclassification rate  can ``blow up'' from $\epsilon_{bad}$ to $\epsilon_{bad}/(\gamma_A-\gamma_B + \epsilon_{bad})$, that is, by a  \textit{multiplicative} factor of  $1/(\gamma_A-\gamma_B + \epsilon_{bad})$ -- even before accounting for finite datasets.

We can now see that DRAFT  both (1) can make non-uniformly-distributed encoding errors, and (2) given as input a dataset with (non-uniform) errors, can produce output ``next-state'' datasets with a \textit{multiplicatively-increased} error rate. Therefore, it seems difficult to derive a sample-complexity analysis of ``recursive DRAFT'' that does not require a number of samples  exponential in the Ex-BMDP horizon $H$.\footnote{We are not claiming that ``Recursive DRAFT'' \textit{actually does} require exponential samples in $H$, simply that there are clear obstacles to proving that it \textit{does not.} } In the next section, we present our CRAFT algorithm, which is intentionally designed to solve the multiple-agent  action-free Ex-BMDP representation learning problem while \textit{avoiding} recursively training state classifiers on datasets derived from the output of previous-timestep state classifiers. We then can sidestep the issues with ``Recursive DRAFT'' shown here.

Note that, for Ex-BMDPs with deterministic latent dynamics, the issues with recursion seen here are unique to the offline, action-free setting. In the online setting, as in \citet{efroni2022provably}, once the dynamics up to timestep $h$ have been learned, ``fresh'' samples of any given latent state $s^*_h$ can then be constructed via closed-loop planning: there are no issues with compounding error.\footnote{Even with \textit{offline} data, if action labels are available and the latent dynamics up to timestep $h$ are known perfectly (which is achievable if the latent dynamics are deterministic), then ``error-free'' datasets can still be constructed  for timestep $h+1$.} The action-free offline setting thus presents a new set of issues requiring a novel algorithmic solution.

\subsection{CRAFT: High-Level Description of Method} \label{sec:craft_description}
\begin{figure}[t]
    \centering
\includegraphics[width=0.945\linewidth]{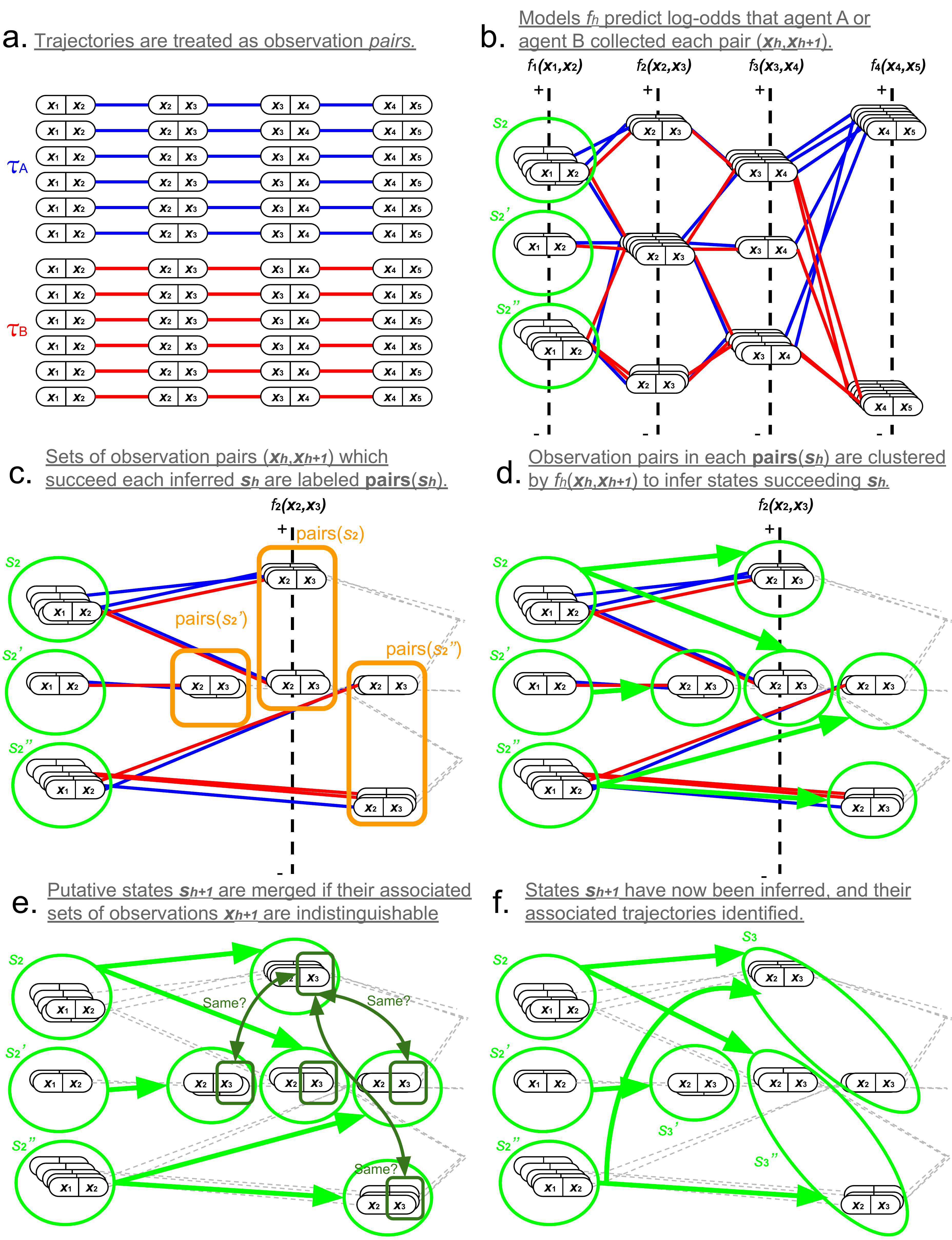}
    \caption{Schematic of the CRAFT algorithm. See text of Section \ref{sec:craft_description}.}
    \label{fig:CRAFT_diagram}
\end{figure}
Here, we give a high-level overview of the CRAFT algorithm. The complete algorithm is presented as Algorithm \ref{alg:main_alg} in Appendix \ref{sec:proofs}. See also Figure \ref{fig:CRAFT_diagram} for a pictorial overview of the approach.

In CRAFT, we initially treat each trajectory as a sequence of observation \textit{pairs}: $(x_1,x_2)$, $(x_2,x_3)$, ... $(x_{H-1},x_H)$. (See Figure \ref{fig:CRAFT_diagram}-a.) For each timestep-pair $(h,h+1)$, we train a model $f_h$ to predict, given a sample $(x_{h},x_{h+1})$, whether the pair was collected by agent $A$ or agent $B$: that is, whether $(x_{h},x_{h+1})$ was selected from $(\tau_A)_{h,h+1}$ or $(\tau_B)_{h,h+1}$. However, unlike in ``DRAFT'', we do not treat this problem as hard binary classification. Instead, we train $f_h(x_{h},x_{h+1})$ to predict the \textit{log-odds ratio} between the two possibilities, for a given $(x_{h},x_{h+1})$. That is, we train $f_h$ to predict 
\begin{equation}
  \ln\left(\frac{\Pr[(x_{h},x_{h+1}) \in (\tau_A)_{h,h+1} |(x_{h},x_{h+1}) \in (\tau_A \uplus \tau_B)_{h,h+1} ]}{\Pr[(x_{h},x_{h+1}) \in (\tau_B)_{h,h+1} |(x_{h},x_{h+1}) \in (\tau_A \uplus \tau_B)_{h,h+1} ]} \right).
\end{equation}

To accomplish this task, we can train $f_h$ to minimize the following loss function:

\begin{equation}
    \mathcal{L}(f_h) := \sum_{(x_h,x_{h+1}) \in (\tau_A)_{h,h+1}} \ln(1 + e^{-f(x_h,x_{h+1})}) + \sum_{(x_h,x_{h+1}) \in (\tau_B)_{h,h+1}} \ln(1 + e^{f(x_h,x_{h+1})}). \label{eq:cluster_loss_main}
\end{equation}
Note that in the limit of infinite data, the $f_h$ that minimizes this loss will return 
\begin{equation}
    f_h^*(x_h,x_{h+1}) \rightarrow \ln\left(\frac{|\mathcal{D}_A^*(\phi^*_h(x_h),\phi^*_{h+1}(x_{h+1}))|}{|\mathcal{D}_B^*(\phi^*_h(x_h),\phi^*_{h+1}(x_{h+1}))|}\right). \label{eq:optimal_clusterer_main}
\end{equation}
Consequently (for sufficiently-large datasets) we expect the values of $f_h(x_h,x_{h+1})$ of all observation-pairs $(x_h,x_{h+1})$ corresponding to the same  latent-state pair $ (\phi^*_h(x_h),\phi^*_{h+1}(x_{h+1})) =(s^*_h,s^*_{h+1})$ to ``cluster together''  around the same value (See Figure \ref{fig:CRAFT_diagram}-b): this effect can be quantified using standard concentration arguments. Note that the training of models $f_h$ and resulting ``clustering'' of observation-pairs can be carried out \textit{simultaneously and independently} for all time-steps $h$: there is no ``recursion'' here, and so each model $f_h$ is trained on an ``untainted'' dataset. 
\begin{tcolorbox}[noparskip]
\hrule
\textbf{Side note on realizability and discretization:} To ensure that an $f_h$ can be found that minimizes Equation \ref{eq:cluster_loss_main}, we need $f_h$ to be chosen from a sufficiently-expressive hypothesis class $\mathcal{F}_h$. We can construct such an $\mathcal{F}_h$ as $\Phi_h \times \Phi_{h+1} \times (N_s^2 \rightarrow \mathbb{R})$: by the realizability assumptions on $\Phi_h$ and $\Phi_{h+1}$, we are ensured that this class contains the optimal predictor in Equation \ref{eq:optimal_clusterer_main}. However, the $(N_s^2 \rightarrow \mathbb{R})$ component of this hypothesis class makes it non-finite. In order to allow for a simple finite-hypothesis analysis, we instead construct $\mathcal{F}_h$ as $\Phi_h \times \Phi_{h+1} \times (N_s^2 \rightarrow \Xi)$, where $\Xi$ is a discrete space (roughly, every $(\alpha/4)$-th interval on a range determined by $\eta$). It turns out that this discretization still ensures that a function ``close enough'' to $f^*_h$ will always exist, and additionally greatly simplifies the identification of ``clusters'' in the output distribution of $f_h$ on finite data.
\hrule
\end{tcolorbox}

While we expect the values of $f_h(x_h,x_{h+1})$ to cluster for sets of  observations-pairs with the \textit{same} latent-state-pair, it does not immediately follow that the values of $f_h(x_h,x_{h+1})$ and $f_h(x_h',x_{h+1}')$ will \textit{differ} if $(s_h,s_{h+1}) \neq (s_h',s_{h+1}')$. In fact, this is not true in general: two distinct ``clusters'' may overlap entirely. This is where the assumption given in Equation \ref{eq:alpha_seperation} becomes useful: from Equation \ref{eq:alpha_seperation} and algebra, we can see that for any fixed $s_h^*$ and distinct  $s_{h+1}^*, s_{h+1}'^*$  which can both follow $s_h^*$ that:
        \begin{equation}
      \left| \ln \left(\frac{ |\mathcal{D}_A^*(s^*_{h}, s^*_{h+1})|}{|\mathcal{D}_B^*(s^*_{h}, s^*_{h+1})| }\right) -  \ln \left(\frac{ |\mathcal{D}_A^*(s^*_{h}, s'^*_{h+1})|}{|\mathcal{D}_B^*(s^*_{h}, s'^*_{h+1})| }\right)\right| \geq \alpha.
    \end{equation}
In other words, the ``clusters'' associated with two pairs $(s^*_{h}, s^*_{h+1})$ and $(s'^*_{h}, s'^*_{h+1})$ are guaranteed to be distinct if $s^*_{h} =s'^*_{h}$. One consequence is that the observation-pairs $(x_1,x_2)$ associated with each possible latent-state pair $(s_1^*,s_2^*)$ must form distinct, well-separated clusters (because these pairs all share the same initial latent state $s_1^*$). Therefore, datasets of observations associated with each latent state $s_2$ can be immediately identified (as shown in green in Figure \ref{fig:CRAFT_diagram}-b; note that we omit the asterisk, to indicate that these are \textit{inferred}, rather than \textit{ground-truth}, latent states.).

CRAFT then continues ``recursively'': once the trajectories which contain a particular $s_{2}' \in \mathcal{S}$ are known, we can then examine the spectrum of values of $f_2(x_2,x_3)$ for \textit{only} observation pairs $(x_2,x_3)$ from this \textit{subset} of trajectories (referred to in Algorithm \ref{alg:main_alg} as pairs($s_2'$)). Because these observation-pairs all (up to an error factor) share the same initial state $s_2'^*$, we expect to see well-separated clusters for each latent state which can succeed $s_2'$. Note that we \textit{do not} retrain $f_2$ on only these samples in pairs($s_2'$). Therefore, any errors (missing or extra trajectories) in the construction of pairs($s_2'$) can only substantially affect the outcome of this step by compromising CRAFT's ability to recognize distinct clusters in the \textit{precomputed values} of $f_2(x_2,x_3)$. Due to the discretization of the range of $\mathcal{F}_h$, this ``cluster identification'' is robust to even adversarial errors affecting a bounded number of trajectories. The total number of misclassified trajectories then grows only linearly with~$H$. (In Figure~\ref{fig:CRAFT_diagram}-c, we show the spectrum of values of $f_2(x_2,x_3)$ for each subset pairs($s_2$), pairs($s_2'$), and pairs($s_2''$);  Figure \ref{fig:CRAFT_diagram}-d shows the result of the cluster identification: the observation-pairs $(x_2,x_3)$ corresponding to each state in $\mathcal{S}_3$ which can succeed each of $s_2,s_2'$, and $s_2''$ have been identified).

Once we identify each latent state that can succeed each $s_2' \in \mathcal{S}_2$ \textit{individually}, we now determine whether or not any of these successor states to distinct states $s_2$, $s_2'$ are in fact \textit{the same} latent state $s_3  \in \mathcal{S}_3$. This can be accomplished easily, by attempting to learn binary classifiers between the observations $x_3$ which are part of the observation-pair sets. If these observations are indistinguishable, then the sets of observation-pairs represent the same latent state; if they are perfectly distinguishable, then they represent different latent states. Figure \ref{fig:CRAFT_diagram}-e illustrates this process. Note that while there may be errors in these observation sets, each binary-classifier training ultimately produces a boolean result (either the sets are distinguishable, or they are not) with a substantial allowance for error in the input sets: there is (with high probability) no accumulation of errors due to this process. 

Finally, the observations corresponding to each unique latent state $\mathcal{S}_3$ have been identified. (See Figure \ref{fig:CRAFT_diagram}-f). We can then continue to timestep $h=4$, and so on. As mentioned above, both the cluster-identification and state-merging processes are robust to bounded errors in their input data, so the total number of misclassified states grows only linearly in $H$. As a final step, the encoders $\phi'_h$ are trained on the assembled datasets for each timestep $h$.

\subsection{Guarantees}
We prove the following polynomial sample-complexity guarantee for CRAFT in Appendix \ref{sec:proofs}:

\begin{restatable}{theorem}{mainthm}
    Assume that CRAFT (Algorithm \ref{alg:main_alg} in the Appendix) is given datasets $\tau_A$ and $\tau_B$ such that the assumptions given in Equations \ref{eq:noise_free_policy_property}, \ref{eq:pair_coverage},\ref{eq:alpha_seperation}, and \ref{eq:eta_coverage} all hold.  
    Then there exists an 
    \begin{equation}
    f\left(H, |\Phi|, N_s,\frac 1\delta, \frac 1{\epsilon_0},\frac 1\nu,\frac 1{\nu'},\frac 1\eta,\frac 1\alpha\right) \in \mathcal{O}^*\left(\frac{H^2(\ln(|\Phi|/\delta) +N_s^2)}{\nu\eta^2\alpha^{4}} \cdot \max\left(\frac{1}{\nu^2}, \frac{1}{\epsilon_0^2\nu'^2}\right)\right), \label{eq:asymptotic_complexity}
        \end{equation} where $\mathcal{O}^*(f(x)) := \mathcal{O}(f(x) \log^k(f(x)))$, such that  for any given $\delta, \epsilon_0 \geq 0$, if 
$\forall s_{h}^*, s_{h+1}^*$ such that  $s_{h}^*$ can transition to $s_{h+1}^*$, 
$|\mathcal{D}^*(s_{h}^*,s_{h+1}^*)| \geq  f\left(H, |\Phi|, N_s,\tfrac 1\delta, \tfrac 1{\epsilon_0},\tfrac 1\nu,\tfrac 1{\nu'},\tfrac 1\eta,\tfrac 1\alpha\right),$ 
then, with probability at least $1-\delta$, the encoders $\phi'_h$ returned by the algorithm will each have accuracy on at least $1-\epsilon_0$, in the sense that, under some bijective mapping $\sigma_h : \mathcal{S}_{h} \rightarrow \mathcal{S}_{h}^*$, 
\begin{equation}
    \forall s^*\in \mathcal{S}_h^*,\,\,\,\Pr_{x\sim \mathcal{Q}(s^*,P^e_h)} (\phi'_h(x) = \sigma_h^{-1}(\phi^*_h(x))) \geq 1-\epsilon_0. \label{eq:conclusion}
\end{equation}
  \label{thm:main_thm}
\end{restatable}
\section{Simulation Results} 
\label{sec:results}
\begin{table}[h!]
    \centering
 \caption{Results of toy environment simulation, with $H=30,M=128$, averaged over 20 random seeds. See text of Section \ref{sec:results}, and Appendix \ref{sec:experiment_details} for further details.}
    \begin{tabular}{|c|c|c|c|}
            \hline
        Technique & Avg. Encoder Acc. $(|\tau_A|  = |\tau_B| =500)$&" " 1000 & " " 5000\\
        \hline
        CRAFT &86.4\%&\textbf{97.7\%}&\textbf{>99.9\%}\\
        Single-obs. classification&67.8\% &68.7\%&69.7\%\\
    Paired-obs. classification&\textbf{87.4\%} &86.1\%&82.1\%\\
        \hline
    \end{tabular}
    \label{tab:results}
\end{table}

We test CRAFT on a toy environment which captures CRAFT's ability to distinguish controllable features in the observation space from time-correlated uncontrollable features.
In the environment, $s_1^* = 0, \mathcal{A} =\mathcal{S}_{h>1}^* = \{0,1\}$ and $s_{h+1}^* = a_h$; in other words, the agent can simply set the next latent state using the action. The exogenous state consists of $M-1$ factors: $e = (e^1,e^2,...,e^{M-1})$. Each exogenous factor is a two-state Markov Chain: for $e^2,...e^{M-1}$, the   initial state distribution and state transition probabilities are arbitrary parameters chosen uniformly at random for each chain, while $e^1$ has $\Pr(e^1_1 = 0) = 0.5$ and transition probabilities of zero.  
The observation $x_h \in \{0,1\}^M$ consists of $s_h^*$ concatenated with $(s_h^* \text{ XOR } e^i_h) $, for each $i\in[H-1]$. Additionally, at each timestep, the order of $s_h^*$ and the other factors is \textit{permuted} by some arbitrary permutation which depends on $h$. The hypothesis classes are $\Phi_h := \{(x_h) \rightarrow (x_h)_i | i \in [M]\}$. The representation learning problem is then to determine, for each $h$, which of the $M$ components of the observation $x_h$ is the controllable factor $s_h^*$ (or, failing at that, to find a component corresponding to a $(s_h^* \text{ XOR } e^i_h)$ where $e^i_h$ is low-entropy, so the encoder imperfect but still useful). Agent $A$ selects actions uniformly at random, while for agent $B$,  $\Pr(a_h=s_h^*) = 3/4$. 

Results are shown in Table \ref{tab:results}. The setting is designed to prevent various ``shortcuts'' to learning an encoder from working. Simply choosing the component of $x_h$ that best predicts the policy ("Single-observation classification" in Table \ref{tab:results}) will not work, because at any sufficiently large timestep $h$, the latent state distributions of the two policies are essentially identical (with a total-variation gap of $2^{-h}$). Furthermore, given observations of a \textit{pair} of sequential timesteps $(x_h,x_{h+1})$, choosing the components of $x_h$ and $x_{h+1}$, respectively, that \textit{together} best predict the agent also will not work ("Paired-observation classification").  In particular, the ``distractor'' features $(s_h^* \text{ XOR } e^1_h)$ and $(s_{h+1} \text{ XOR } e^1_{h+1})$ are, taken together, about as informative about the \textit{agent's identity} as $s_h^*$ and $s_{h+1}^*$, but provide \textit{no} information about the latent state $s_h^*$~or~$s_{h+1}^*$. In Table \ref{tab:results}, we see that, given sufficient data ($\geq 1000$ trajectories for each agent), CRAFT is capable of learning highly-accurate encoders in this setting, while these two ``shortcut'' techniques are not. In particular, while the ``Paired-observation classification" shortcut is about as effective as CRAFT in the very-low data regime, its performance plateaus (and even seems to drop) as more data becomes available. (The \textit{drop} in performance is likely because the adversarially-designed ``distractor'' features $(s_h^* \text{ XOR } e^1_h)$ and $(s_{h+1}^* \text{ XOR } e^1_{h+1})$ are more likely to be chosen by this method as more data becomes available.)

\section{Related Works} 
\label{sec:related_works}
\textbf{Action-free representation learning} Many prior works have tackled action-free representation learning in practical scenarios, demonstrating empirically-validated methods. Common approaches utilize observation reconstruction losses \citep{seo2022reinforcement} or temporal contrastive losses \citep{nair2023r3m}. Some of these works infer ``latent actions'' by finding a compact representation that is highly informative for predicting forward dynamics \citep{edwards2019imitating,menapace2021playable,ye2023become,schmidt2024learning} Another line of work augments large action-free offline datasets with significantly smaller action-labeled datasets. For example, an offline dataset action-free dataset can be used to train a goal-conditioned value function \citep{xu2022a,ma2023vip,ghosh2023reinforcement,park2023hiql}. Alternatively, an inverse-dynamics model can be learned from the action-labelled data to ``fill in'' missing actions \citep{pmlr-v155-schmeckpeper21a,zheng2023semi,baker2022video}. By contrast, in this work we are interested in provable sample-efficiency of representation learning, and assume no access to action-labeled data during pretraining. 

\textbf{Learning in Ex-BMDPs.} As discussed throughout this work, numerous prior works consider the Ex-BMDP model \citep{efroni2022provably, mhammedi2024the}, including in the offline setting \citep{islam2023principled,lamb2023guaranteed,levine2024multistep}. \cite{misra2024towards} in particular demonstrates a hardness result: that Ex-BMDP latent representations cannot be learned in general from offline action-free data. In this work, we demonstrate a special case where this representation learning problem is in fact tractable: the case where offline data from multiple diverse agents are available.
\section{Discussion and Limitations} \label{sec:discussion}
One major assumption of this work (as well as \citet{misra2024towards,islam2023principled} and other prior works) is that offline data are collected by a policy which acts independently of observation noise.
This assumption stems from the fact that, if noise features influence the behavioral policy, they (indirectly) influence the latent-state dynamics of the agent: these noise features may then be erroneously captured in the learned representation. However, in real-world settings, it may actually be beneficial to capture such features in the learned representation: if ``expert'' agents are relying on some uncontrollable feature, this feature may be relevant to the expert agents' reward functions, and may therefore also be relevant to downstream tasks for which our learned representations will be used. Therefore,  the noise-independent policy assumption might not be necessary in practice.

An additional restrictive assumption of this work is that the latent dynamics are deterministic, and that each episode starts at the same latent state $s_1^*$. However, this assumption is also essentially present even in the best-known result for provably sample-efficient Ex-BMDP representation learning in the online setting \citep{efroni2022provably} -- that work does allow for \textit{rare} departures from deterministic dynamics, however, and it may be possible to adapt the analysis of CRAFT to that setting as well, although we have focused on the strictly-deterministic case here for ease of presentation. \cite{mhammedi2024the} proposes an online algorithm for learning  Ex-BMDPs with nondeterministic latent dynamics, but that work assumes ``simulator access'': the ability to reset the environment to any previously-visited observation. Several works \citep{lamb2023guaranteed,levine2024multistep,islam2023principled} consider learning Ex-BMDPs from offline data (with action labels) without assuming restarts to $s_1^*$: these works are ``practical'' algorithms that do not provide sample-complexity guarantees. A similar ``practical'' algorithm for the action-free, multiple-agent setting based on the ideas presented in this work may also be possible.

The assumption that two policies differ substantially at \textit{every} latent state may also be impractical. One direction for future work may be to leverage data from \textit{several} agents, such that it is more likely that \textit{some} agent has a distinct behavior at each latent state.

While access to training oracles is a common assumption in representation learning \citep{agarwal2020flambe,efroni2022provably,uehara2022representation}, the optimization of Equation \ref{eq:cluster_loss_main}, on a \textit{discretized} domain, may be  troublesome in practice. Additionally, the sample complexity bounds in Equation \ref{eq:asymptotic_complexity}, while polynomial, may not be optimal: these issues are potential directions for future work.
\section*{Acknowledgments}
A portion of this research has taken place in the Learning Agents Research Group (LARG) at the Artificial Intelligence Laboratory, The University of Texas at Austin.  LARG research is supported in part by the National Science Foundation (FAIN-2019844, NRT-2125858), the Office of Naval Research (N00014-18-2243), Army Research Office (W911NF-23-2-0004, W911NF-17-2-0181), DARPA (Cooperative Agreement HR00112520004 on Ad Hoc Teamwork), Lockheed Martin, and Good Systems, a research grand challenge at the University of Texas at Austin.  The views and conclusions contained in this document are those of the authors alone.  Peter Stone serves as the Executive Director of Sony AI America and receives financial compensation for this work. The terms of this arrangement have been reviewed and approved by the University of Texas at Austin in accordance with its policy on objectivity in research. Alexander Levine is supported by the NSF Institute for Foundations of Machine Learning (FAIN-2019844). Amy Zhang and Alexander Levine are supported by National Science Foundation (2340651) and Army Research Office (W911NF-24-1-0193).
\bibliography{main}
\bibliographystyle{rlj}

\appendix
\onecolumn
\section{Hypothesis Classes and Realizability Assumptions} \label{sec:aaprealizability}
As mentioned in Section \ref{sec:cpdr}, we assume access to hypothesis classes of encoder functions $\Phi_{1:H}$.  We make a realizability assumption: that is, the true encoder $\phi^*_h \in \Phi_h$. Moreover, we assume that for any arbitrary permutation $\sigma$, $\sigma(\phi^*_h(x)) \in \Phi_h$ -- this allows us to train an encoder in $\Phi_h$ on datasets of observations representing each latent state without knowing the ``correct'' ordering of the latent states.  Similar realizability assumptions are common in representation learning literature for structured MDPs \citep{du2019provably,efroni2022provably,misra2020kinematic,misra2024towards,uehara2022representation,agarwal2020flambe}. 

We use a second set of  hypothesis classes,  $\mathcal{G}_h \subseteq \mathcal{X}_h \rightarrow \{0,1\}$, for which, for any \textit{pair} of latent states $s^*_h,s'^*_h$, there exists some $g \in \mathcal{G}_h$ that can perfectly distinguish observations of $s^*_h$ from observations of $s'^*_h$. In our sample-complexity results, we assume that $|\mathcal{G}_h| \leq |\Phi_h|.$
In this work, we are chiefly concerned with sample-complexity: we make use of training oracles for a variety of loss functions which may not be tractable to optimize in practice. See Section \ref{sec:discussion} for further discussion.
\section{Algorithm} \label{sec:appendix_alg}

The full CRAFT algorithm is presented as Algorithm \ref{alg:main_alg}.
\begin{algorithm}
\begin{algorithmic}[1]
{\small
   \REQUIRE Trajectory datasets $\tau_A$, $\tau_B$, known lower-bounds $\alpha$, $\eta$, and $\nu$, encoder function classes $\Phi_h \subseteq \mathcal{X}_h \rightarrow N_s$ and  
   classification function class $\mathcal{G}_h\subseteq \mathcal{X}_h \rightarrow \{0,1\}$.
    \STATE{ $\alpha \leftarrow \min(1,\alpha)$.}
    \STATE{Let $\xi := \alpha/4$; $n_{\Xi}:=   \lceil 8\ln(\eta^{-1}-1)/\alpha\rceil $.}
    \STATE{ $\eta \leftarrow 1/(1+e^{n_{\Xi}\alpha/8})$.}
    \STATE{Let $\Xi := \{-\frac{n_{\Xi}\xi}{2}, -\frac{n_{\Xi}\xi}{2} + \xi, -\frac{n_{\Xi}\xi}{2} + 2\xi, ..., \frac{n_{\Xi}\xi}{2} \}.$}
   
   \STATE Initialize $\mathcal{S}_1 := \{s_1\}$,
   $D_{A,1}(s_1) := [|\tau_A|]$, $D_{B,1}(s_1) := [|\tau_B|]$.
\texttt{\hspace{1em}\textbackslash\textbackslash ~First timestep should have a single state rep., associated with every trajectory index.}
  \STATE{Let $\phi_1' := \mathcal{X}_1 \rightarrow 0$}.
   \FOR{$h \in \{1,2,...,H-1\}$} 
    \STATE Initialize $\mathcal{S}_{h+1} := \{\}$.
     \STATE {Let the inverse-actor-prediction function class $\mathcal{F}_h\subseteq \mathcal{X} \times \mathcal{X} \rightarrow \Xi$ be composed as 
   $\mathcal{F}_h = \Phi_h \times  \Phi_{h+1}\times ( N_s^2 \rightarrow \Xi)$.} 
   \STATE Find the $f_h \in \mathcal{F}_h$ which minimizes Equation \ref{eq:cluster_loss_main}.
   \STATE{Let $q_{thresh.} := \frac{h\nu}{8H} $.}
   \FOR{$s_{pred} \in \mathcal{S}_h$}
   \STATE{Initialize $\text{merged\_already}(s) := \text{False}$ for all $s \in \mathcal{S}_{h+1}$.}
   \STATE{Initialize $\mathcal{S}_{new} := \{\}$}
   \STATE Let $\text{pairs}(s_{pred}) := (\tau_A)_{h:\,h+1}[D_{A,h} (s_{pred})]\cup (\tau_B)_{h:\,h+1}[D_{B,h}(s_{pred})]$. \texttt{\hspace{1em}\textbackslash\textbackslash ~ All transitions which start at $s_{prev}$.}
    \STATE  $\forall j \in \{0,...,n_\Xi\}$, Let $\text{pred\_succ}[j] := \{(x_{h},x_{h+1}) \in \text{pairs}(s_{pred})|f(x_{h},x_{h+1}) = j\xi - \frac{n_\Xi\xi}{2}\}$.
   \STATE Initialize $j := 0$
  \WHILE{$j \leq n_\Xi$}
 
 \IF{$|\text{pred\_succ}[j] |\geq q_{thresh.} (|\tau_A| + |\tau_B|)$}
  \STATE{Let $j'$ be the minimum integer $> j$ such that $|\text{pred\_succ}[j'] | < q_{thresh.}(|\tau_A| + |\tau_B|)$, or $n_\Xi$ if no such integer exists.}
  \STATE{Let $\mathcal{D}_{new\_pairs} := \{x'| (x,x') \in \biguplus_{k = \max(0,j-1)}^{j'} \text{pred\_succ}[k]\},
  \,\mathcal{D}_{new} := \{x'| (x,x') \in \mathcal{D}_{new\_pairs}\}$}
  \STATE{Initialize new\_state? $\leftarrow $ True.}
  \FOR{$s \in \mathcal{S}_{h+1}$, such that merged\_already($s$) $==$ False}
  \STATE{Let $\mathcal{D}_s := (\tau_A)_{h+1}[D_{A,h+1} (s)]\uplus (\tau_B)_{h+1}[D_{B,h+1}(s)]$}
  \STATE{Train a classifier $g \in \mathcal{G}$ to distinguish $\mathcal{D}_{new}$ and  $\mathcal{D}_s$, with loss $\mathcal{L}(g)$ given in Equation \ref{eq:binary_lassification}. }
 \IF{the loss $\mathcal{L}(g)$ on $\mathcal{D}_{new}$ and $\mathcal{D}_s$ is  $> 0.5$ }
  \STATE{Append to $D_{A,h+1}(s)$ indices of trajectories in $\tau_A$ that observations in $\mathcal{D}_{new}$ are from.}
  \STATE{Append to $D_{B,h+1}(s)$ indices of trajectories in $\tau_B$ that observations in $\mathcal{D}_{new}$ are from.}
  \STATE{merged\_already?$(s)\leftarrow$ True}
  \STATE{new\_state? $\leftarrow $ False}
  \STATE{\textbf{break.}}
 \ENDIF
 \ENDFOR
 \IF{new\_state?}
  \STATE{Add new state $s_{new}$ to $\mathcal{S}_{new}$}
  \STATE{Initialize $D_{A,h+1}(s_{new})$ as indices of trajectories in $\tau_A$ that observations in $\mathcal{D}_{new}$ are from.}
  \STATE{Initialize $D_{B,h+1}(s_{new})$ as indices of trajectories in $\tau_B$ that observations in $\mathcal{D}_{new}$ are from.}
 \ENDIF
  \STATE{$j \leftarrow j' +2$}
 \ELSE
 \STATE{$j \leftarrow j+1$}
 \ENDIF
  \ENDWHILE
  \STATE{$\mathcal{S}_{h+1} := \mathcal{S}_{h+1} \cup  \mathcal{S}_{new}$}
  \ENDFOR
 \STATE{$ \phi_{h+1}' := \arg\min_{\phi\in \Phi_{h+1}}  \sum_{s\in \mathcal{S}_{h+1}}\left[\frac{1}{|\mathcal{D}_s|}\sum_{x\in \mathcal{D}_s} (1-\mathbbm{1}_{(\phi(x)=s)}) \right]$, where $\mathcal{D}_s := (\tau_A)_{h+1}[D_{A,{h+1}} (s)]\uplus (\tau_B)_{h+1}[D_{B,{h+1}}(s)]$} 
  \ENDFOR
  \STATE{ \textbf{Return: }$\phi_1', ...\phi_H'$ }
  }
\end{algorithmic}
\caption{CRAFT} \label{alg:main_alg}
\end{algorithm}
\section{Proofs}
\label{sec:proofs}
In this section, we prove the correctness and sample complexity bounds of CRAFT presented in Theorem \ref{thm:main_thm}. First, though, we prove various lemmas the will be helpful in proving the final result.
\subsection{Preliminary Note}
Recall Equation \ref{eq:noise_free_policy_property} in the main text:
\begin{equation*}
\begin{split}
  \Pr(\tau_A, \tau_B) = &\Pr(\phi^*(\tau_A), \phi^*(\tau_B)) \cdot  \Pr_{P_1^e,\mathcal{T}^e}(\phi^{e}(\tau_A))\cdot \Pr_{P_1^e,\mathcal{T}^e}( \phi^{e}(\tau_B))\\ 
\cdot & \Pr_Q(\tau_A|\phi^*(\tau_A),\phi^{e}(\tau_A) 
) \cdot \Pr_Q(\tau_B|\phi^*(\tau_B),\phi^{e}(\tau_B) 
)
\end{split}
\end{equation*}
Throughout our proofs, we will make  use of this assumption in  the following way: we will treat the controllable latent state trajectories $\phi^*(\tau_A), \phi^*(\tau_B)$ as \textit{fixed but arbitrary}, not as random variables, and treat the exogenous noise Markov chains $\phi^{e}(\tau_A), \phi^{e}(\tau_B)$ and the emission function $\mathcal{Q}$ as the only random variables. Then, if the algorithm succeeds with high probability for any such \textit{fixed, arbitrary} $\phi^*(\tau_A), \phi^*(\tau_B)$, we can conclude by the independence assumption that it also succeeds with high probability under any data-generating process for which Equation \ref{eq:noise_free_policy_property} holds. 
\subsection{Concentration Lemmas}
In this section, we present concentration bounds on the loss functions used in Algorithm \ref{alg:main_alg}. We start with the log-odds loss given in Equation \ref{eq:cluster_loss_main}:
\begin{lemma} \label{lem:spread_log_prob_2}. 
Given $m$ distributions $\mathcal{D}_1,...,\mathcal{D}_m \in \mathcal{P}(\mathcal{X})$, each with two corresponding positive integers $a_i,b_i$, for $i \in [m]$, let $A_i \sim \mathcal{D}^{a_i}$ and $B_i \sim \mathcal{D}^{b_i}$ be two multi-sets consisting of $a_i$ and $b_i$ i.i.d. samples from $\mathcal{D}_i$, respectively.
Then, for any $\xi > 0$ and $n_\Xi \in \mathbb{N}_+$ such that $\forall i, \,|\ln(a_i/b_i)| \leq \frac{n_\Xi\xi}2$, let $\Xi = \{-\frac{n_\Xi\xi}2, -\frac{n_\Xi\xi}2 +\xi ,  -\frac{n_\Xi\xi}2 +2\xi, ..., \frac{n_\Xi\xi}2 \}$.  Further, let $\bar{c}_i \in \Xi$ be the smallest value in $\Xi$ greater than or equal to $\ln(a_i/b_i)$, and $\underline{c}_i\in \Xi$ be the largest value in $\Xi$ less than or equal to $\ln(a_i/b_i)$. Also, assume that $\forall i, \frac{a_i+b_i}{\sum_{i'=1}^m a_{i'} + b_{i'}} \geq \nu$.

Given any function $f \in \mathcal{X} \rightarrow \Xi$, define: 
\begin{equation}
    \mathcal{L}(f) := \sum_{i = 1}^m \left[\sum_{x \in A_i}\ln(e^{-f(x)} +1) + \sum_{x \in B_i} \ln(e^{f(x)} +1)\right].\label{eq:def_loss_v2}
\end{equation}
Further, define:
\begin{equation}
    \mathcal{L}_{\text{ref}} :=  \sum_{i = 1}^m \left[\min_{c_i\in \{\bar{c}_i, \underline{c}_i\}} a_i\ln\left(e^{-c_i} + 1\right)  +b_i\ln\left(e^{c_i}+1\right)\right],
\end{equation}
and let $\eta := (e^{n_\Xi \xi /2} +1)^{-1}$. For any $\epsilon$ and $\delta$, if:
\begin{equation}
  \forall i \in [m] ,\,\,\,a_i + b_i \geq \frac{50 \ln(2/\delta) \ln^2(1/\eta)}{\nu\epsilon^2\eta^2 \xi^4} \label{eq:lemm_num_samples_v2}
\end{equation}
then the probability that \textbf{both}  $ \mathcal{L} \leq \mathcal{L}_{\text{ref}} $, \textbf{and}:
\begin{equation}
    \exists i \in [m]:\,\,\,\Big|\{x \in A_i \uplus B_i| f(x)\not \in \{\underline{c}_i,\bar{c}_i\}\Big| > \epsilon (a_i+b_i)\label{eq:spectrum_concentrarion_conclusion_v2}
\end{equation}
is at most $\delta.$
\end{lemma}
\begin{proof}
We can define
\begin{equation}
    \mathcal{L}_{\text{ref}}^i :=  \min_{c_i\in \{\bar{c}_i, \underline{c}_i\}} a_i\ln\left(e^{-c_i} + 1\right)  +b_i\ln\left(e^{c_i}+1\right), \label{eq:c_und_c_bar_v2}
\end{equation}
So that  $\mathcal{L}_{\text{ref}} = \sum_{i = 1}^m \mathcal{L}_{\text{ref}}^i$, and also define 
\begin{equation}
           \mathcal{L}_{pop}^i(f) :=a_i\left(  \mathop{\mathbb{E}}_{x\sim \mathcal{D}_i}\ln(e^{-f(x)} +1)\right) + b_i\left(  \mathop{\mathbb{E}}_{x\sim \mathcal{D}_i}\ln(e^{f(x)} +1)\right) \label{eq:def_pop_loss_v2}
\end{equation}
and $\mathcal{L}_{\text{pop}} := \sum_{i = 1}^m \mathcal{L}_{\text{pop}}^i$.

First, we consider the      ``population loss'' for each distribution:
\begin{equation}
\begin{split}
    \mathcal{L}_{pop}^i(f) &=a_i\left(  \mathop{\mathbb{E}}_{x\sim \mathcal{D}_i}\ln(e^{-f(x)} +1)\right) + b_i\left(  \mathop{\mathbb{E}}_{x\sim \mathcal{D}_i}\ln(e^{f(x)} +1)\right) \\
       &=a_i\left(  \sum_{\zeta\in \Xi} \Pr_{x\sim \mathcal{D}_i}(f(x) = \zeta)\cdot\ln(e^{-\zeta} +1)\right) \\
       &\hspace{1em}+ b_i\left(  \sum_{\zeta\in \Xi} \Pr_{x\sim \mathcal{D}_i}(f(x) = \zeta)\cdot\ln(e^{\zeta} +1)\right)\\
       &= \sum_{\zeta\in \Xi} \Pr_{x\sim \mathcal{D}_i}(f(x) = \zeta)\cdot
       \left( a_i\ln(e^{-\zeta} +1) +  b_i\ln(e^{\zeta} +1)\right)\\
       &= \sum_{\zeta\in \Xi} \Pr_{x\sim \mathcal{D}_i}(f(x) = \zeta)\cdot
       a_i\left(  \left(1 + \frac {b_i}{a_i} \right)\ln(e^\zeta +1) - \zeta\right) \label{eq:def_pop_loss}
\end{split}
\end{equation}
We can define $h_\gamma(\zeta):= \left(1 + \gamma^{-1} \right)\ln(e^\zeta +1) - \zeta $, so that 
\begin{equation}
       \mathcal{L}_{pop}^i(f) = a_i\sum_{\zeta\in \Xi} \Pr_{x\sim \mathcal{D}_i}(f(x) = \zeta)\cdot  h_{^{a_i}\!/\!_{b_i}}(\zeta) \label{eq:l_pop}
\end{equation}
Now, note that:
\begin{equation}
    h_\gamma'(\zeta)= \left(1 + \gamma^{-1} \right)\frac{e^\zeta}{e^\zeta + 1} - 1
\end{equation}
and
\begin{equation}
    h_\gamma''(\zeta)= \left(1 + \gamma^{-1} \right)\frac{e^\zeta}{(e^\zeta + 1)^2}.
\end{equation}
Then, we see that $h_\gamma(\zeta)$ is a convex function, with a global minimum at $\zeta = \ln(\gamma)$, and second-derivative at least 
\begin{equation}
    \left(1+\gamma^{-1}\right)\frac{e^{n_\Xi\xi/2}}{(e^{n_\Xi\xi/2}+1)^2} \,\,\left(= \left(1+\gamma^{-1}\right)\frac{e^{-n_\Xi\xi/2}}{(e^{-n_\Xi\xi/2}+1)^2}\right)
\end{equation}
everywhere on the interval $[-n_\Xi\xi/2,n_\Xi\xi/2]$.

Due to the convexity of $h_{^{a_i}\!/\!_{b_i}}(\zeta)$, we have, $\forall j >0$,
\begin{equation*}
    a_i \cdot h_{^{a_i}\!/\!_{b_i}}(\ln(a_i/b_i)) \leq \min( a_i \cdot  h_{^{a_i}\!/\!_{b_i}}(\bar{c}_i),  a_i \cdot  h_{^{a_i}\!/\!_{b_i}}(\underline{c}_i))\, (= \mathcal{L}_{ref}^i) \leq  a_i \cdot  h_{^{a_i}\!/\!_{b_i}}(\bar{c}_i) <  a_i \cdot h_{^{a_i}\!/\!_{b_i}}(\bar{c}_i +j\xi)
\end{equation*}
and, similarly, $\forall j>0$:
\begin{equation*}
    a_i \cdot h_{^{a_i}\!/\!_{b_i}}(\ln(a_i/b_i)) \leq \min( a_i \cdot  h_{^{a_i}\!/\!_{b_i}}(\bar{c}_i),  a_i \cdot  h_{^{a_i}\!/\!_{b_i}}(\underline{c}_i))\, (= \mathcal{L}_{ref}^i) \leq  a_i \cdot  h_{^{a_i}\!/\!_{b_i}}(\underline{c}_i) <  a_i \cdot h_{^{a_i}\!/\!_{b_i}}(\underline{c}_i -j\xi).
\end{equation*}
In particular, by a second-order Taylor bound, we have that, for $j > 0$:
\begin{equation}
\begin{split}
      \mathcal{L}_{ref}^i &\leq a_i \cdot h_{^{a_i}\!/\!_{b_i}}(\bar{c}_i +j\xi) - a_i\cdot \left(1+(a_i/b_i)^{-1}\right)\frac{e^{n_\Xi\xi/2}}{(e^{n_\Xi\xi/2}+1)^2} \cdot \frac{(j\xi)^2}2 \\
     & \leq   a_i \cdot h_{^{a_i}\!/\!_{b_i}}(\bar{c}_i +j\xi) - \left(a_i+b_i\right)\frac{e^{n_\Xi\xi/2}}{(e^{n_\Xi\xi/2}+1)^2} \cdot \frac{\xi^2}2  
\end{split}
\end{equation}
and similarly for $ a _i\cdot h_{^{a_i}\!/\!_{b_i}}(\underline{c}_i -j\xi)$:
\begin{equation}
    \mathcal{L}_{ref}^i  \leq   a_i \cdot h_{^{a_i}\!/\!_{b_i}}(\underline{c}_i -j\xi) - \left(a_i+b_i\right)\frac{e^{n_\Xi\xi/2}}{(e^{n_\Xi\xi/2}+1)^2} \cdot \frac{\xi^2}2.
\end{equation}
In particular, by Equation \ref{eq:l_pop},
\begin{equation}
     \mathcal{L}_{pop}^i(f) \geq \mathcal{L}_{ref}^i +  \Pr_{x\sim \mathcal{D}_i}(f(x) \not \in \{\underline{c}_i, \bar c_i\}) \cdot\left(a_i+b_i\right)\frac{e^{n_\Xi\xi/2}}{(e^{n_\Xi\xi/2}+1)^2} \cdot \frac{\xi^2}2.
\end{equation}
In terms of $\eta$, this is:
\begin{equation}
\begin{split}
       \mathcal{L}_{pop}^i(f) &\geq \mathcal{L}_{ref}^i +  \Pr_{x\sim \mathcal{D}_i}(f(x) \not \in \{\underline{c}_i, \bar c_i\}) \cdot\left(a_i+b_i\right) (\eta-\eta^2) \cdot \frac{\xi^2}2 \\
       &\geq \mathcal{L}_{ref}^i +  \Pr_{x\sim \mathcal{D}_i}(f(x) \not \in \{\underline{c}_i, \bar c_i\}) \cdot\left(a_i+b_i\right) \cdot \frac{\eta\xi^2}4  
\end{split}
\end{equation}
where we use the fact that $\eta \leq 1/2$ in the last inequality.
This gives us:
\begin{equation}
\Pr_{x\sim \mathcal{D}_i}(f(x) \not \in \{\underline{c}_i, \bar c_i\})  \leq  \frac{4( \mathcal{L}_{pop}^i(f)-\mathcal{L}_{ref}^i) }{(a_i+b_i)\eta \xi^2} \label{eq:pop_ref_gap}
\end{equation}
Because $\mathcal{L}_{pop}^i - \mathcal{L}_{ref}^i \geq 0$, this implies:
\begin{equation}
\begin{split}
   &\forall i \in [m],\\&(a_i+b_i) \cdot \Pr_{x\sim \mathcal{D}_i}(f(x) \not \in \{\underline{c}_i, \bar c_i\})  \leq  \frac{4( \mathcal{L}_{pop}^i(f)-\mathcal{L}_{ref}^i) }{\eta \xi^2} \leq \frac{4\sum_{i=1}^m ( \mathcal{L}_{pop}^i(f)-\mathcal{L}_{ref}^i) }{\eta \xi^2} \\
   &\hspace{13em}= \frac{4( \mathcal{L}_{pop}(f)-\mathcal{L}_{ref}) }{\eta \xi^2}  \label{eq:pop_ref_gap_v2}
\end{split}
\end{equation}

Meanwhile, from Equations \ref{eq:def_loss_v2} and \ref{eq:def_pop_loss_v2} applying (one-sided) Hoeffding's lemma  gives us, with probability at least $1-\delta/2$:

\begin{equation}
    \mathcal{L}(f)    -\mathcal{L}_{pop}(f) +  \sqrt{\sum_{i=1}^m(a_i+b_i)} \ln(1/\eta)\sqrt{2\ln(2/\delta)} \geq 0.
\end{equation}

which implies, by assumption:
\begin{equation}
   \forall i,\,\, \mathcal{L}(f)    -\mathcal{L}_{pop}(f) +  \sqrt{a_i+b_i} \sqrt{1/\nu}\ln(1/\eta)\sqrt{2\ln(2/\delta)} \geq 0. \label{eq:l_pop_gap_v2}
\end{equation}

Combining Equations \ref{eq:pop_ref_gap_v2} and \ref{eq:l_pop_gap_v2} gives, with probability at least $1-\delta/2$, we have that $\forall i \in [m],$
\begin{equation}
(a_i+b_i)\Pr_{x\sim \mathcal{D}_i}(f(x) \not \in \{\underline{c}_i, \bar c_i\})  \leq  \frac{4( \mathcal{L}(f)-\mathcal{L}_{ref}) }{\eta \xi^2} + \frac{4(\sqrt{a_i+b_i})\ln(1/\eta)\sqrt{2\ln(2/\delta)}}{\sqrt{\nu}\eta \xi^2}.
\end{equation}
Then, with probability at least $1-\delta/2$, the condition $\mathcal{L}(f)\leq \mathcal{L}_{ref}$ implies:
\begin{equation}
\forall i \in [m],\,\,\,
(a_i+b_i)\Pr_{x\sim \mathcal{D}_i}(f(x) \not \in \{\underline{c}_i, \bar c_i\})  \leq  \frac{4\sqrt{a_i +b_i} \ln(1/\eta)\sqrt{2\ln(2/\delta)}}{\sqrt{\nu}\eta \xi^2}.
\end{equation}
We can apply Hoeffding's lemma \textit{once for each} $i \in [m]$, to the binary variable of whether on not $f(x) \in \{\underline{c}_i,\bar{c}_i\}$, where $x$ is sampled $(a_i+b_i)$ times to produce the dataset $A_i \cup B_i$. By union bound, we have, with probability at least $1-\delta$, $\mathcal{L}(f)\leq \mathcal{L}_{ref}$ implies:
\begin{equation}
\begin{split}
     \forall i \in [m],\,\,\, \Big|\{x \in A_i \cup B_i| f(x)\not \in \{\underline{c}_i,\bar{c}_i\}\Big| &\leq \\\sqrt{(a_i+b_i)\ln(2m/\delta)/2} + \frac{4\sqrt{a_i+b_i}\ln(1/\eta)\sqrt{2\ln(2/\delta)}}{\sqrt{\nu}\eta \xi^2}
       &\leq \\\sqrt{(a_i+b_i)}\sqrt{2}\left(\frac{\sqrt{\ln(2m/\delta)}}{2} + \frac{4 \sqrt{\ln(2/\delta)}\ln(1/\eta)}{\sqrt{\nu}\eta \xi^2}\right)
       &\leq \\\sqrt{(a_i+b_i)}\sqrt{2}\frac{ \ln(1/\eta)}{\eta \xi^2}\left(\frac{\sqrt{\ln(2/\delta)}}{2} + \frac{\sqrt{\ln(m)}}{2} + \frac{4 \sqrt{\ln(2/\delta)}}{\sqrt{\nu}}\right),\label{eq:multi_lemma_xi_nu_1}
\end{split}
\end{equation}

where in the last line, we used triangle inequality and the fact that  $\eta  = 1/(e^{n_{\Xi}\xi/2} + 1) \leq 1/(e^{\xi/2} + 1) $, which in turn implies:
\begin{equation}
    \frac{\ln(1/\eta)}{\eta \xi^2} \geq \frac{(e^{\xi/2} + 1)\ln(e^{\xi/2} + 1)}{\xi^2} > 1 \,\,\,\,\,(\forall \xi > 0). \label{eq:xi_eta_1_bound}
\end{equation}

Note that, because \text{each} $a_i + b_i$ contains at least a $\nu$-fraction of the total $\sum_i^m a_i +b_i$, we must have $m \leq 1/\nu$. Then:
\begin{equation}
    \frac{\sqrt{\ln(m)}}{2} \leq   \frac{\sqrt{\ln(1/\nu)}}{2} \leq \frac{1}{2\sqrt{e}\sqrt\nu} \leq \frac{1}{2\sqrt{e}\sqrt{\ln(2)}}\frac{\sqrt{\ln(2/\delta)}}{\sqrt\nu} \leq \frac{\sqrt{\ln(2/\delta)}}{2\sqrt\nu}
\end{equation}
Therefore (and noting $\nu < 1)$, we can combine terms in Equation \ref{eq:multi_lemma_xi_nu_1} to conclude:
\begin{equation}
     \forall i \in [m],\,\,\, \Big|\{x \in A_i \cup B_i| f(x)\not \in \{\underline{c}_i,\bar{c}_i\}\Big| 
       \leq 
  \frac{5 \sqrt{2(a_i+b_i)\ln(2/\delta)} \ln(1/\eta)}{\sqrt{\nu}\eta \xi^2}.
\end{equation}

Now, to ensure $  \forall i \in [m],\,\,\,\Big|\{x \in A_i \cup B_i| f(x)\not \in \{\underline{c}_i,\bar{c}_i\}\Big| \leq \epsilon(a_i+b_i)$, we need,
\begin{equation}
 \forall i \in [m],\,\,\,
\frac{5 \sqrt{2(a_i+b_i)\ln(2/\delta)} \ln(1/\eta)}{\sqrt{\nu}\eta \xi^2}\leq \epsilon(a_i+b_i)
\end{equation}
or:
\begin{equation}
 \forall i \in [m],\,\,\,
\frac{50 \ln(2/\delta) \ln^2(1/\eta)}{\nu\epsilon^2\eta^2 \xi^4} \leq a_i+b_i
\end{equation}
as provided by Equation \ref{eq:lemm_num_samples_v2}. Note that because the implication 
\begin{equation}
  \mathcal{L}(f)\leq \mathcal{L}_{ref} \rightarrow \forall i \in [m],\,\,\, \Big|\{x \in A_i \cup B_i| f(x)\not \in \{\underline{c}_i,\bar{c}_i\}\Big| \leq \epsilon(a_i+b_i)  
\end{equation}
 holds with probability at least $(1-\delta)$, this implication can only be \textit{broken}, by the case that 
\begin{equation}
  \mathcal{L}(f)\leq \mathcal{L}_{ref} \land  \exists i \in [m]:\,\Big|\{x \in A_i \cup B_i| f(x)\not \in \{\underline{c}_i,\bar{c}_i\}\Big| > \epsilon (a_i+b_i)
\end{equation}
with probability at most $\delta$.
\end{proof}
\begin{corollary}
Let $\mathcal{D}^*(s_{h}^*,s_{h+1}^*)$ be the multiset of observation pairs $(x_{h},x_{h+1})$ from both $\tau_A$ and $\tau_B$ in Algorithm 1, such that $\phi^*_h(x_{h}) = s_{h}^*$ and $\phi^*_{h+1}(x_{h+1}) = s_{h+1}^*$, and let $\mathcal{D}_A^*(s_{h}^*,s_{h+1}^*)$ and $\mathcal{D}_B^*(s_{h}^*,s_{h+1}^*)$ be the elements of $\mathcal{D}^*(s_{h}^*,s_{h+1}^*)$ originating from $\tau_A$ and $\tau_B$ respectively.  Further, let $\bar{c}_{s^*_{h},s^*_{h+1}} \in \Xi$ be the smallest value in $\Xi$ greater than or equal to $\ln(|\mathcal{D}_A^*(s_{h}^*,s_{h+1}^*)|/|\mathcal{D}_B^*(s_{h}^*,s_{h+1}^*)|)$, and $\underline{c}_{s^*_{h},s^*_{h+1}}\in \Xi$ be the largest value in $\Xi$ less than or equal to $\ln(|\mathcal{D}_A^*(s_{h}^*,s_{h+1}^*)|/|\mathcal{D}_B^*(s_{h}^*,s_{h+1}^*)|)$. 

Further, assume the   realizability condition that $ \phi^*_h \in \Phi_h$ and  $ \phi^*_{h+1} \in \Phi_{h+1}$.

If $\forall s_{h}^*, s_{h+1}^*, \text{such that } s_{h}^* \text{ can transition to }s_{h+1}^*,$
\begin{equation}
       |\mathcal{D}^*(s_{h}^*,s_{h+1}^*)| \geq \frac{50 (\ln(2|\Phi|^2/\delta) + N_s^2 \ln(n_\Xi+1)) \ln^2(1/\eta)}{\nu\epsilon^2\eta^2 \xi^4}
\end{equation}
then with probability at least $1-\delta$, the function $f(\cdot)$ found in Line 10 of Algorithm \ref{alg:main_alg} will be such that, $\forall s_{h}^*, s_{h+1}^*, \text{such that } s_{h}^* \text{ can transition to }s_{h+1}^*,$
\begin{equation}
    \Big|\{x \in \mathcal{D}^*(s_{h}^*,s_{h+1}^*)| f(x)\not \in \{\underline{c}_{s^*_{h},s^*_{h+1}},\bar{c}_{s^*_{h},s^*_{h+1}}\}\Big| \leq \epsilon  \Big|\mathcal{D}^*(s_{h}^*,s_{h+1}^*) \Big| .\label{eq:spread_bound_corr}
\end{equation} \label{cor:spread_bound}
\end{corollary}
\begin{proof}
By application of Lemma \ref{lem:spread_log_prob_2} with $\delta' :=  \delta/(|\Phi|^2\cdot (n_\Xi+1)^{(N_s^2)} ) \leq \delta/|\mathcal{F}_h| $, we have that, for any fixed hypothesis $f'$, the probability that $\mathcal{L}(f') \leq \mathcal{L}_{ref}$ and Equation \ref{eq:spread_bound_corr} is violated is at most $\delta/|\mathcal{F}_h|$. Then by union bound, the probability that any such $f'$ exists in $\mathcal{F}_h$ is at most $1-\delta$. However, by the realizability assumption, we know that an $f^*$ exists in $\mathcal{F}$ which achieves loss  $\mathcal{L}(f^*) = \mathcal{L}_{ref}$ and also that respects Equation \ref{eq:spread_bound_corr}. (In particular, this $f^*$ is simply $(\phi^*_h, \phi^*_{h+1})$ composed with a mapping from the representations corresponding to each $(s^*_{h},s^*_{h+1})$ to the corresponding $\underline{c}_{s^*_{h},s^*_{h+1}}$ or $\bar{c}_{s^*_{h},s^*_{h+1}}$ which minimizes Equation \ref{eq:c_und_c_bar_v2}.) Therefore with probability at least $1-\delta$, the $f \in \mathcal{F}_h$ which minimizes $\mathcal{L}(f)$ must respect Equation \ref{eq:spread_bound_corr}.
\end{proof}
We now give two simple results for classification under corrupted data. First though, we prove a minor claim, which is simply some ``deferred algebra'' for the lemmas which follow:

\begin{proposition} \label{prop:algebra}
    Consider a multiset $\mathcal{Z} = \{z_1,...,z_m\}$ of items $z_i\in [0,1]$, and a modified multiset $\mathcal{Z}'$, also consisting of items in $[0,1]$, such that the symmetric difference between $\mathcal{Z} $ and $\mathcal{Z}'$ has size at most $k$ (that is, $\mathcal{Z}'$ can be constructed from $\mathcal{Z} $ by inserting and/or removing a total of at most $k$ items). Then \begin{equation}
        \left|\sum_{\mathcal{Z}} \frac{z}{|\mathcal{Z}|} -\sum_{\mathcal{Z'}} \frac{z'}{|\mathcal{Z'}|}  \right| \leq \frac{k}{m}
    \end{equation}
\end{proposition}
\begin{proof}
 Define $\mathcal{Z}_{removed}, \mathcal{Z}_{added},$ and  $\mathcal{Z}_{kept}$ such that $\mathcal{Z} = \mathcal{Z}_{kept} + \mathcal{Z}_{removed}$, and $\mathcal{Z}' = \mathcal{Z}_{kept} + \mathcal{Z}_{added}$. Note that $k = |\mathcal{Z}_{removed}| + |\mathcal{Z}_{added}|$ and  $m = |\mathcal{Z}_{removed}| + |\mathcal{Z}_{kept}|$.
We first assume that $|\mathcal{Z}| \geq |\mathcal{Z}'| $. In other words, we assume $|\mathcal{Z}_{removed}| \geq |\mathcal{Z}_{added}|$. Then we can construct $\mathcal{Z}' $ from $\mathcal{Z}$ by (1) removing some arbitrary subset $\mathcal{Z}_{removed}' \subseteq \mathcal{Z}_{removed}$ from $\mathcal{Z}$, such that $|\mathcal{Z}_{removed}'|$ = $|\mathcal{Z}_{added}|$; then (2) inserting the samples $\mathcal{Z}_{added}$; and then finally (3) removing the multiset $\mathcal{Z}_{removed}'' = \mathcal{Z}_{removed} \setminus \mathcal{Z}_{removed}'$. Let the intermediate set constructed after  step (2)  be $\mathcal{Z}'' := (\mathcal{Z}\setminus \mathcal{Z}_{removed}') \uplus \mathcal{Z}_{added} = \mathcal{Z}' \uplus \mathcal{Z}_{removed}''$. Note that $|\mathcal{Z}| = |\mathcal{Z}''|$, and 
  \begin{equation}
    \begin{split}
&\left|\sum_{\mathcal{Z}} \frac{z}{|\mathcal{Z}|} -\sum_{\mathcal{Z}''} \frac{z''}{|\mathcal{Z}''|}  \right| = \left|\sum_{\mathcal{Z}} \frac{z}{|\mathcal{Z}|} -\sum_{\mathcal{Z}''} \frac{z''}{|\mathcal{Z}|}  \right| = \\
&\frac{1}{|\mathcal{Z}|} \left|\sum_{\mathcal{Z}} z -\sum_{\mathcal{Z}''} z'' \right| =  \frac{1}{|\mathcal{Z}|} \left|\sum_{\mathcal{Z'}_{removed}} z -\sum_{\mathcal{Z}_{added}} z \right| \leq \frac{|\mathcal{Z}_{added}|}{|\mathcal{Z}|} 
\end{split}
\end{equation}
 Additionally,  note that:
    \begin{equation}
    \begin{split}
&\left|\sum_{\mathcal{Z}''} \frac{z''}{|\mathcal{Z}''|} -\sum_{\mathcal{Z}'} \frac{z'}{|\mathcal{Z}'|}  \right| =
        \frac{1}{|\mathcal{Z}''|}\left|\sum_{\mathcal{Z}''} z'' -\frac{|\mathcal{Z}''|\sum_{\mathcal{Z}'} z'}{|\mathcal{Z}'|}  \right| = \\& \frac{1}{|\mathcal{Z}|}\left|\frac{|\mathcal{Z}'|\sum_{\mathcal{Z}'} z'}{|\mathcal{Z}'|} +\frac{|\mathcal{Z}_{removed}''|\sum_{\mathcal{Z}_{removed}''} z_r}{|\mathcal{Z}_{removed}''|} -\frac{|\mathcal{Z}''|\sum_{\mathcal{Z}'} z'}{|\mathcal{Z}'|}  \right|  =\\
        & \frac{1}{|\mathcal{Z}|}\left|\frac{|\mathcal{Z}_{removed}''|\sum_{\mathcal{Z}_{removed}''} z_r}{|\mathcal{Z}_{removed}''|} -\frac{|\mathcal{Z}_{removed}''|\sum_{\mathcal{Z}'} z'}{|\mathcal{Z}'|}  \right|  =\\
&\frac{|\mathcal{Z}_{removed}''|}{|\mathcal{Z}|}\left|\frac{\sum_{\mathcal{Z}_{removed}''} z_r}{|\mathcal{Z}_{removed}''|} -\frac{\sum_{\mathcal{Z}'} z'}{|\mathcal{Z}'|}  \right|  \leq \frac{|\mathcal{Z}_{removed}''|}{|\mathcal{Z}|}
    \end{split}
    \end{equation}
    Finally, by triangle inequality, we have that 
    \begin{equation}
    \begin{split}
&\left|\sum_{\mathcal{Z}} \frac{z}{|\mathcal{Z}|} -\sum_{\mathcal{Z}'} \frac{z'}{|\mathcal{Z}'|}  \right|    \leq \left|\sum_{\mathcal{Z}} \frac{z}{|\mathcal{Z}|} -\sum_{\mathcal{Z}''} \frac{z''}{|\mathcal{Z}''|}  \right| +\left|\sum_{\mathcal{Z}''} \frac{z''}{|\mathcal{Z}''|} -\sum_{\mathcal{Z}'} \frac{z'}{|\mathcal{Z}'|}  \right| \\
&\leq \frac{|\mathcal{Z}_{added}|}{|\mathcal{Z}|} +\frac{|\mathcal{Z}_{removed}''|}{|\mathcal{Z}|} \leq \frac{|\mathcal{Z}_{added}|}{|\mathcal{Z}|} +\frac{|\mathcal{Z}_{removed}|}{|\mathcal{Z}|} \leq \frac{k}{m}
    \end{split}
    \end{equation}
 as desired. A similar argument can be made for the case of $|\mathcal{Z}| <|\mathcal{Z}'| $.
\end{proof}

We now give the classification lemmas:
\begin{lemma} \label{lemma:binary_classification}
Given a finite hypothesis class $\mathcal{G}\subseteq \mathcal{X} \rightarrow \{0,1\}$, and two datasets (multisets) $A, B \subset \mathcal{X}$, let:
\begin{equation}
\begin{split}
       g' &:= \arg\min_{g\in \mathcal{G}} \mathcal{L}(g)\\
    \mathcal{L}(g)&:=\frac{1}{|A|}\sum_{a\in A} (1-g(a))\,+\,\frac{1}{|B|}\sum_{b\in B} g(b). \label{eq:binary_lassification}
\end{split}
\end{equation}
Let $n = \min(|A|,|B|)$, and assume that $A$ and $B$ are constructed as follows:
\begin{itemize}
    \item $A' \sim \mathcal{D}_A^{|A'|}$
    \item $B' \sim \mathcal{D}_B^{|B'|}$
    \item At most a total of $m$ arbitrary (non-i.i.d.) samples are either added to or removed from $A'$ or $B'$, or moved from $A'$ to $B'$ or vice-versa, to create $A$ and $B$.
\end{itemize}
Then, if
\begin{equation}
    n \geq 8m\text{ and } n \geq \frac{128}{7} \ln(2|\mathcal{G}|/\delta)
\end{equation}
then with probability at least $1-\delta$,
\begin{itemize}
    \item If $\mathcal{D}_A = \mathcal{D}_B$, then $\mathcal{L}(g') > 1/2$
    \item Conversely, if $\mathcal{D}_A$ and $\mathcal{D}_B$ have disjoint support, such that some $g^*\in \mathcal{G}$ maps all elements in the support of $\mathcal{D}_A$ to 1 and all elements in the support of $\mathcal{D}_B$ to 0, then $\mathcal{L}(g') \leq 1/2$.
\end{itemize}
\end{lemma}

\begin{proof}
    Define 
\begin{equation}
    \mathcal{L}_{clean}(g):=\frac{1}{|A'|}\sum_{a\in A'} (1-g(a))\,+\,\frac{1}{|B'|}\sum_{b\in B'} g(b). 
\end{equation}
    Then, fix any $g \in \mathcal{G}$. From some algebra (see Proposition \ref{prop:algebra}), it can be shown that
    \begin{equation}
         \mathcal{L}_{clean}(g) -\frac{2m}{n} \leq  \mathcal{L}(g) \leq  \mathcal{L}_{clean}(g)  +\frac{2m}{n} \label{eq:binary_lemma_dirty}
    \end{equation}
Now, note that   $\mathcal{L}_{clean}$ is the sum of $|A'|$ random variables bounded on $[0,1/|A'|]$, and $|B'|$ random variables bounded on $[0,1/|B'|]$, all of which are i.i.d. Then, by Hoeffding's Lemma and Equation \ref{eq:binary_lemma_dirty}, with probability $1-\delta/|\mathcal{G}|$:
\begin{equation}
\begin{split}
  \mathbb{E}[\mathcal{L}_{clean}(g)]-  \sqrt{\left(\frac{1}{|A'|} + \frac{1}{|B'|}\right)\ln(2|\mathcal{G}|/\delta)/2} -\frac{2m}{n} < \mathcal{L}_{clean}(g) -\frac{2m}{n}  &\leq \mathcal{L}(g)\\ \leq  \mathcal{L}_{clean}(g)  +\frac{2m}{n} < \mathbb{E}[\mathcal{L}_{clean}(g)]+  \sqrt{\left(\frac{1}{|A'|} + \frac{1}{|B'|}\right)\ln(2|\mathcal{G}|/\delta)/2} +\frac{2m}{n}&
  \end{split}
\end{equation}
Because $|A'|, |B'| \geq n-m$, and applying union bound over all $g \in \mathcal{G}$, we have, with probability $1-\delta$:
\begin{equation}
\forall g \in \mathcal{G} ,\,\, \mathbb{E}[\mathcal{L}_{clean}(g)]-  \sqrt{\frac{\ln(2|\mathcal{G}|/\delta)}{n-m}} -\frac{2m}{n}    < \mathcal{L}(g) <  \mathbb{E}[\mathcal{L}_{clean}(g)]+  \sqrt{\frac{\ln(2|\mathcal{G}|/\delta)}{n-m}} +\frac{2m}{n}. \label{eq:binary_test_ineq}
\end{equation}
Note that:
\begin{equation}
   \forall g\in \mathcal{G},\, \mathbb{E}[\mathcal{L}_{clean}(g)] =  1 - \mathbb{E}_{x \in \mathcal{D}_A}[g(x)] + \mathbb{E}_{x \in \mathcal{D}_B}[g(x)]. 
\end{equation}

If $\mathcal{D}_A = \mathcal{D}_B$, then $\forall g\in \mathcal{G},\,\mathbb{E}[\mathcal{L}_{clean}(g)] =1$, so by Equation \ref{eq:binary_test_ineq}, we have, with probability $1-\delta$:
\begin{equation}
\forall g\in \mathcal{G},\,  1-  \sqrt{\frac{\ln(2|\mathcal{G}|/\delta)}{n-m}} -\frac{2m}{n}    \leq \mathcal{L}(g) , 
\end{equation}
and in particular:
\begin{equation}
  1-  \sqrt{\frac{\ln(2|\mathcal{G}|/\delta)}{n-m}} -\frac{2m}{n}    < \mathcal{L}(g'). 
\end{equation}
Conversely, if $\mathcal{D}_A$ and $\mathcal{D}_B$ have disjoint support, such that some $g^*\in \mathcal{G}$ maps all elements in the support of $\mathcal{D}_A$ to 1 and all elements in the support of $\mathcal{D}_B$ to 0, then we have:
\begin{equation}
    \mathbb{E}[\mathcal{L}_{clean}(g^*)] =  1 - \mathbb{E}_{x \in \mathcal{D}_A}[g^*(x)] + \mathbb{E}_{x \in \mathcal{D}_B}[g^*(x)] = 1-1-0 = 0. 
\end{equation}
Then, with probability at least $1-\delta$:
\begin{equation}
    \mathcal{L}(g') \leq \mathcal{L}(g^*) < \sqrt{\frac{\ln(2|\mathcal{G}|/\delta)}{n-m}} +\frac{2m}{n}.
\end{equation}
To complete the proof, we only need to show that 
\begin{equation}
\sqrt{\frac{\ln(2|\mathcal{G}|/\delta)}{n-m}} +\frac{2m}{n} \leq \frac{1}{2}.
\end{equation}
With $m \leq n/8$, this condition becomes:
\begin{equation}
\sqrt{\frac{8\ln(2|\mathcal{G}|/\delta)}{7n}}  \leq \frac{1}{4}.
\end{equation}
or 
\begin{equation}
n \geq \frac{128}{7} \ln(2|\mathcal{G}|/\delta)
\end{equation}
\end{proof}
\begin{lemma}
Given a finite hypothesis class $\Phi\subset\mathcal{X} \rightarrow \mathbb{N}$, and N datasets (multisets) $D_1,D_2,...D_N \subset \mathcal{X}$, let:
\begin{equation}
\begin{split}
       \phi' &:= \arg\min_{\phi\in \Phi} \mathcal{L}(\phi)\\
    \mathcal{L}(\phi)&:= \sum_{i=1}^N\left[\frac{1}{|D_i|}\sum_{x\in D_i} (1-\mathbbm{1}_{(\phi(x)=i)}) \right]
\end{split}
\end{equation}
Let $n = \min_i(|D_i|)$, and assume that each $D_i$ is constructed as follows:
\begin{itemize}
    \item $\forall i,\,\,\,D_i' \sim \mathcal{D}_i^{|D_i'|}$
    \item At most a total of $m$ arbitrary (non-i.i.d.) samples are arbitrarily moved between the datasets $D_i'$, to create the datasets $D_i$.
\end{itemize}
Additionally, assume that $\exists \phi^*\in \Phi: \forall i,\,x\sim \mathcal{D}_i \implies \phi^*(x)= i$. Then, if
\begin{equation}
    n \geq \frac{8m}{\epsilon}\text{ and } n\geq \frac{64N\ln(2|\Phi|/\delta)}{7\epsilon^2} 
\end{equation}
then with probability at least $1-\delta$,
\begin{equation}
    \forall i\in [N],\,\,\,\Pr_{x\sim \mathcal{D}_i} (\phi'(x) = i) \geq 1-\epsilon.  \label{eq:multi_class_lemma_target}
\end{equation} 
\label{lem:multi_class}
\end{lemma}
\begin{proof}
    Define 
\begin{equation}
    \mathcal{L}_{clean}(\phi):=\sum_{i=1}^N\left[\frac{1}{|D_i'|}\sum_{x\in D_i'} (1-\mathbbm{1}_{(\phi(x)=i)}) \right].
\end{equation}
    Then, fix any $\phi \in \Phi$. From Proposition \ref{prop:algebra} (regarding each transfer of a sample as removing a sample into one multiset $D'_i$, and inserting a new sample into another)  we see that:
    \begin{equation}
    \mathcal{L}_{clean}(\phi) -\frac{2m}{n} \leq  \mathcal{L}(\phi) \leq  \mathcal{L}_{clean}(\phi)  +\frac{2m}{n} \label{eq:lemma_dirty}
    \end{equation}
Now, note that   $\mathcal{L}_{clean}$ is the sum of $|D_1'|$ random variables bounded on $[0,1/|D_1'|]$, and $|D_2'|$ random variables bounded on $[0,1/|D_2'|]$, et cetera, all of which are i.i.d. Then, by Hoeffding's Lemma and Equation \ref{eq:lemma_dirty}, with probability $1-\delta/|\Phi|$:
\begin{equation}
\begin{split}
  \mathbb{E}[\mathcal{L}_{clean}(\phi)]-  \sqrt{\left( \sum_{i\in [N]}\frac{1}{|D_i'|}\right)\ln(2|\Phi|/\delta)/2} -\frac{2m}{n} < \mathcal{L}_{clean}(\phi) -\frac{2m}{n}  &\leq \mathcal{L}(\phi)\\ \leq  \mathcal{L}_{clean}(\phi)  +\frac{2m}{n} < \mathbb{E}[\mathcal{L}_{clean}(\phi)]+  \sqrt{\left( \sum_{i\in [N]}\frac{1}{|D_i'|}\right)\ln(2|\Phi|/\delta)/2} +\frac{2m}{n}&
  \end{split}
\end{equation}
Because $\forall i, |D_i'| \geq n-m$, and applying union bound over all $\phi \in \Phi$, we have, with probability $1-\delta$, $\forall \phi \in \Phi,\,\,$:
\begin{equation}
 \mathbb{E}[\mathcal{L}_{clean}(\phi)]-  \sqrt{\frac{N\ln(2|\Phi|/\delta)}{2(n-m)}} -\frac{2m}{n}    < \mathcal{L}(\phi) <  \mathbb{E}[\mathcal{L}_{clean}(\phi)]+  \sqrt{\frac{N\ln(2|\Phi|/\delta)}{2(n-m)}} +\frac{2m}{n}. \label{eq:multi_test_ineq}
\end{equation}
Then we have:
\begin{equation}
\begin{split}
     \mathbb{E}[\mathcal{L}_{clean}(\phi')] < \mathcal{L} (\phi') + \sqrt{\frac{N\ln(2|\Phi|/\delta)}{2(n-m)}} +\frac{2m}{n}   \leq\\  \mathcal{L} (\phi^*) +  \sqrt{\frac{N\ln(2|\Phi|/\delta)}{2(n-m)}} +\frac{2m}{n}  <\mathbb{E}[\mathcal{L}_{clean}(\phi^*)]+  \sqrt{\frac{2N\ln(2|\Phi|/\delta)}{(n-m)}} +\frac{4m}{n} = \\\sqrt{\frac{2N\ln(2|\Phi|/\delta)}{(n-m)}} +\frac{4m}{n}.
     \label{eq:multi_lemma_bound}
\end{split}
\end{equation}
Where we use the fact that, by the definition of $\phi'$ as a minimizer, $\mathcal{L}(\phi') \leq \mathcal{L}(\phi^*) $, as well as the fact that, by definition, $\mathcal{L}_{clean}(\phi^*) = 0$.

Also, note that by the definition of $\mathcal{L}_{clean}$, we have that, for any $\phi$, 
\begin{equation}
  \mathbb{E}[\mathcal{L}_{clean}(\phi)] = \sum_{i\in [N]}   1 - \Pr_{x\sim \mathcal{D}_i}(\phi(x) = i) \
\end{equation}
Then, for any particular $i\in [N],$ we have that $ 1 - \Pr_{x\sim \mathcal{D}_i}(\phi(x) = i) \leq \mathbb{E}[\mathcal{L}_{clean}(\phi)]$. Then, by Equation \ref{eq:multi_lemma_bound}, we have, $\forall i \in [N]$:
\begin{equation}
    1 - \Pr_{x\sim \mathcal{D}_i}(\phi'(x) = i)  < \sqrt{\frac{2N\ln(2|\Phi|/\delta)}{(n-m)}} +\frac{4m}{n}.
\end{equation}

By algebra, our desired result (Equation \ref{eq:multi_class_lemma_target}) holds as long as:
\begin{equation} \sqrt{\frac{2N\ln(2|\Phi|/\delta)}{(n-m)}} +\frac{4m}{n} \leq \epsilon
\end{equation}

Which follows from the given conditions on $n$.
\end{proof}

\subsection{Main Proof of Theorem \ref{thm:main_thm} }
Here, we present the proof of Theorem \ref{thm:main_thm}. We first split out correctness proof of the main recursive step of the algorithm as a lemma:
\begin{lemma}
     In Algorithm \ref{alg:main_alg}, 
    suppose that the ground-truth data coverage assumptions given in Equations \ref{eq:pair_coverage},\ref{eq:alpha_seperation}, and \ref{eq:eta_coverage} all hold. Additionally, assume that the relative coverage lower-bound $\eta$ can be written in the form 
    \begin{equation}
        \eta = \frac{e^{-n_{\Xi}\alpha/8}}{1 + e^{-n_{\Xi}\alpha/8}}
    \end{equation}
    for some non-negative integer $n_{\Xi}$.  Further, assume that for each $s^* \in \mathcal{S}^*_{h}$, there exists some $s \in \mathcal{S}_{h}$ that represents approximately the same set of observations. In particular, each index in $[|\tau_A|]$ appears in at most one set $D_{A,h}(s)$ for some $s$ (and likewise for $[|\tau_B|]$ and  $D_{B,h}(s)$), and there exists some bijective mapping $\sigma_h: \mathcal{S}_{h} \rightarrow \mathcal{S}_{h}^*$, such that for most indices $j$ in $[|\tau_A|]$ 
    \begin{equation}
        j \in D_{A,h}(\sigma^{-1}_h(\phi^*_h((\tau_{A})_{h}[j])))  
    \end{equation}
    and for most indices $j$ in $[|\tau_B|]$ 
    \begin{equation}
        j \in D_{B,h}(\sigma^{-1}_h(\phi^*_h((\tau_{B})_{h}[j]))),
    \end{equation}  
    with at most a combined $\beta(|\tau_A| + |\tau_B| )$ indices in either dataset for which this does not hold.

For any $\epsilon$ such that:
\begin{equation}
        \epsilon  <\frac{\nu}{8} - \beta \label{eq:eps_nu_beta_condtion_1},
\end{equation}
and 
\begin{equation}
        \epsilon  + \beta\leq q_{thresh.}<\frac{\nu(1-\epsilon)}{2} -\beta \label{eq:q_thresh_condition},
\end{equation}
where $q_{thresh.}$ is the threshold defined on Line 11 of the algorithm,
assume that $\forall s_{h}^*, s_{h+1}^*, \text{such that } s_{h}^* \text{ can transition to }s_{h+1}^*,$
\begin{equation}
       |\mathcal{D}^*(s_{h}^*,s_{h+1}^*)| \geq \frac{12800 (\ln(4|\Phi|^2/\delta) + N_s^2 \ln(n_\Xi+1)) \ln^2(1/\eta)}{\nu\epsilon^2\eta^2 \alpha^4}. \label{eq:required_samples_recursive}
\end{equation}
Then, with high probability, for each $s^* \in \mathcal{S}^*_{h+1}$, there exists some $s \in \mathcal{S}_{h+1}$ that represents approximately the same set of observations.  In particular, each index in $[|\tau_A|]$ appears in at most one set $D_{A}(s)$ for some $s$ (and likewise for $[|\tau_B|]$ and  $D_{B}(s)$), and there exists some bijective mapping $\sigma_{h+1}: \mathcal{S}_{h+1} \rightarrow \mathcal{S}_{h+1}^*$, such that for most indices $j$ in $[|\tau_A|]$ 
    \begin{equation}
        j \in D_{A,h+1}(\sigma^{-1}_{h+1}(\phi^*_{h+1}((\tau_{A})_{h+1}[j])))  
    \end{equation}
    and for most indices $j$ in $[|\tau_A|]$ 
    \begin{equation}
        j \in D_{B,h+1}(\sigma^{-1}_{h+1}(\phi^*_{h+1}((\tau_{B})_{h+1}[j]))),
    \end{equation}  
    with at most a combined $(\beta+\epsilon)(|\tau_A| + |\tau_B| )$ indices in either dataset for which this does not hold. \label{lem:recursive_lemma}
\end{lemma}
\begin{proof}
    We first show that the datasets of observation pairs $\mathcal{D}_{new\_pairs}$ defined in Line 21 of the algorithm each correspond uniquely to a pair of ground truth latent states in $\mathcal{S}_{h}^* \times \mathcal{S}_{h+1}^*$, such that no pair of observations is included in more than one such $\mathcal{D}_{new\_pairs}$ sets, and, with high probability, each pair of observations $x,x'$ is included in the correct $\mathcal{D}_{new\_pairs}$ corresponding to $(\phi^*_h(x),\phi^*_{h+1}(x'))$, with up to at most $(\beta+\epsilon)(|\tau_A| + |\tau_B| )$ exceptions. 

    For any $s_{pred}^* \in \mathcal{S}_{h}^*$, consider any two distinct $s^*,s'^* \in \mathcal{S}_{h+1}^*$, such that $| \mathcal{D}^*(s^*_{pred}, s^*)| , | \mathcal{D}^*(s^*_{pred}, s'^*)| > 0$. 

    Recall the assumption that, without loss of generality,
    \begin{equation}
       e^\alpha \cdot \frac{  \pi^{emp.}_B(s'^*| s^*_{pred})}{\pi^{emp.}_B(s^*| s^*_{pred})} \leq \frac{ \pi^{emp.}_A(s'^*| s^*_{pred})}{\pi^{emp.}_A(s^*| s^*_{pred}) },
    \end{equation}
    Multiplying both sides by $\pi^{emp.}_A(s^*| s^*_{pred})/\pi^{emp.}_B(s'^*| s^*_{pred})$ yields
    \begin{equation}
       e^\alpha \cdot \frac{ \pi^{emp.}_A(s^*| s^*_{pred})}{\pi^{emp.}_B(s^*| s^*_{pred})} \leq \frac{ \pi^{emp.}_A(s'^*| s^*_{pred})}{\pi^{emp.}_B(s'^*| s^*_{pred})},
    \end{equation}
    From the definition of $\pi^{emp.}$, this is:
    \begin{equation}
       e^\alpha \cdot \frac{ |\mathcal{D}_A^*(s^*_{pred}, s^*)| / |\mathcal{D}_A^*(s^*_{pred})| }{|\mathcal{D}_B^*(s^*_{pred}, s^*)| / |\mathcal{D}_B^*(s^*_{pred})| } \leq \frac{ |\mathcal{D}_A^*(s^*_{pred}, s'^*)| / |\mathcal{D}_A^*(s^*_{pred})| }{|\mathcal{D}_B^*(s^*_{pred}, s'^*)| / |\mathcal{D}_B^*(s^*_{pred})| }.
    \end{equation}
    Multiplying both sides by $|\mathcal{D}_A^*(s^*_{pred})|/|\mathcal{D}_B^*(s^*_{pred})|$ and taking the logarithms yields:
        \begin{equation}
       \alpha + \ln \left(\frac{ |\mathcal{D}_A^*(s^*_{pred}, s^*)|}{|\mathcal{D}_B^*(s^*_{pred}, s^*)| }\right) \leq \ln \left(\frac{ |\mathcal{D}_A^*(s^*_{pred}, s'^*)|}{|\mathcal{D}_B^*(s^*_{pred}, s'^*)| }\right).
    \end{equation}
By the definitions of $\underline{c}_{s^*_{h},s^*_{h+1}}$ and $\bar{c}_{s^*_{h},s^*_{h+1}}$ in Corollary \ref{cor:spread_bound}, and the fact that $\xi = \alpha/4$, we see that there must be at least two values in $\Xi$ between $\bar{c}_{s^*_{pred},s^*}$ and  $\underline{c}_{s^*_{pred},s'^*}$; that is to say: 
\begin{equation}
  \underline{c}_{s^*_{pred},s^*}  \leq \bar{c}_{s^*_{pred},s^*} < \bar{c}_{s^*_{pred},s^*} + \xi < \underline{c}_{s^*_{pred},s'^*} - \xi < \underline{c}_{s^*_{pred},s'^*}  \leq \bar{c}_{s^*_{pred},s'^*} \label{eq:xis_nonoverlapping}
\end{equation}

Therefore, by Corollary \ref{cor:spread_bound} we have, with probability at least $1-\delta/2$, for any $s^*$ such that  $|\mathcal{D}^*(s^*_{pred}, s^*)| > 0$:
\begin{itemize}
    \item At least $(1-\epsilon)|\mathcal{D}^*(s^*_{pred}, s^*)|$ of the samples in $\mathcal{D}^*(s^*_{pred}, s^*)$ will be  mapped by $f$ to $\underline{c}_{s^*_{pred},s^*}$ or  $\bar{c}_{s^*_{pred},s^*}$
    \item for some choice of $\hat{c}_{s^*_{pred},s^*} \in \{\underline{c}_{s^*_{pred},s^*}, \bar{c}_{s^*_{pred},s^*} \}$, at least $(1-\epsilon)/2 \cdot |\mathcal{D}^*(s^*_{pred}, s^*)|$ of the samples in $\mathcal{D}^*(s^*_{pred}, s^*)$ will be mapped to $\hat{c}_{s^*_{pred},s^*}$.
    \item By definition, $\{\underline{c}_{s^*_{pred},s^*}, \bar{c}_{s^*_{pred},s^*} \} \subset \{\hat{c}_{s^*_{pred},s^*} - \xi, \hat{c}_{s^*_{pred},s^*}, \hat{c}_{s^*_{pred},s^*} + \xi \}$.
    \item Furthermore, for no two states $s^*$, $s'^*$, with $\hat{c}_{s^*_{pred},s^*} \in \{\underline{c}_{s^*_{pred},s^*}, \bar{c}_{s^*_{pred},s^*} \}$ and $\hat{c}_{s^*_{pred},s'^*} \in \{\underline{c}_{s^*_{pred},s'^*}, \bar{c}_{s^*_{pred},s'^*} \}$ chosen arbitrarily, will the sets $\{\hat{c}_{s^*_{pred},s^*} - \xi, \hat{c}_{s^*_{pred},s^*}, \hat{c}_{s^*_{pred},s^*} + \xi \}$ and $\{\hat{c}_{s^*_{pred},s'^*} - \xi, \hat{c}_{s^*_{pred},s'^*}, \hat{c}_{s^*_{pred},s'^*} + \xi \}$ overlap (By Equation \ref{eq:xis_nonoverlapping}).
    \item Recall that by assumption, $ |\mathcal{D}^*(s^*_{pred}, s^*)|  \geq \nu (|\tau_A| + |\tau_B|)$. Therefore, at least $(\nu(1-\epsilon)/2) (|\tau_A| + |\tau_B|)$ of the samples in $\mathcal{D}^*(s^*_{pred}, s^*)$ will be mapped to $\hat{c}_{s^*_{pred},s^*} $. 
    \item The total number of samples in $\mathcal{D}^*(s^*_{pred}, s'^*)$, over \textit{all} choices of $ s'^*$, which are \textit{not} mapped by $f$ to a value in the respective set $\{\hat{c}_{s^*_{pred},s'^*} - \xi, \hat{c}_{s^*_{pred},s'^*}, \hat{c}_{s^*_{pred},s'^*} + \xi \}$, is at most $\epsilon|\mathcal{D}^*(s^*_{pred})|$.
    \item $\epsilon|\mathcal{D}^*(s^*_{pred})| \leq \epsilon (|\tau_A| + |\tau_B|).$
\end{itemize}
Therefore, as long as $(\nu(1-\epsilon)/2) > \epsilon$, then among the pairs in $\mathcal{D}^*(s^*_{pred},\cdot) := \uplus_{s'^*} \mathcal{D}^*(s^*_{pred},s'^*)$, if there is any $z\in \Xi$ such that $> \epsilon (|\tau_A| + |\tau_B|)$ of the pairs are mapped by $f$ to $z $, then we know that the set of elements in $\mathcal{D}^*(s^*_{pred},\cdot) $ which are mapped to $\{z-1,z,z+1\}$ contains at least $(1-\epsilon)$ of the elements of the set  $\mathcal{D}^*(s^*_{pred},s^*)$ for some $s^*$; furthermore, such a $z$ exists for each possible value of $s^*$ where  $|\mathcal{D}^*(s^*_{pred},s^*)| > 0$, and, for distinct $s^*$ and  $s'^*$, these values ($\{z-\xi,z,z+\xi\}$ and $\{z'-\xi,z',z'+\xi\}$) are non-overlapping. Consequently, by identifying subsets of $\mathcal{D}^*(s^*_{pred},\cdot) $ of size greater than  $\epsilon(|\tau_A| + |\tau_B|)$ that $f$ maps to the same value, and expanding these subsets to the elements in  $\mathcal{D}^*(s^*_{pred},\cdot)$ mapped to adjacent values in $\Xi$, we can partition $\mathcal{D}^*(s^*_{pred},\cdot)$ into subsets corresponding to each $\mathcal{D}^*(s^*_{pred},s^*)$, with at most $\epsilon|\mathcal{D}^*(s^*_{pred},\cdot)|$ errors.

Note however that we do not have access to $\mathcal{D}^*(s^*_{pred},\cdot)$, only to $\text{pairs}(s_{pred})$ (where $\sigma_h(s_{pred}) = s_{pred}^*$). However, by assumption, $\mathcal{D}^*(s^*_{pred},\cdot)$ and $\text{pairs}(s_{pred})$ differ (in terms of symmetric difference) by at most $\beta (|\tau_A|+|\tau_B|)$. Therefore, we claim that, if 
\begin{equation}
    \epsilon +  \beta \leq q_{thresh.} < \frac{\nu(1-\epsilon)}2  -\beta \label{eq:eps_nu_beta_condtion}
\end{equation}
then, we can identify values of  $\hat{c}_{s^*_{pred},s'^*}$ (for some $s'^*$) as those values $j\xi-\frac{n_\Xi\xi}{2}$ for which (as shown in Line 20 of Algorithm \ref{alg:main_alg}):
\begin{equation}
    \text{pred\_succ}[j]  > q_{thresh.}(|\tau_A| +|\tau_B|),
\end{equation}
and, conversely, if 
\begin{equation}
\text{pred\_succ}[j]  \leq  q_{thresh.}(|\tau_A| +|\tau_B|),
\end{equation}
then $j\xi-\frac{n_\Xi\xi}{2}$ does not correspond to some 
$\bar{c}_{s^*_{pred},s'^*}$ or $\underline{c}_{s^*_{pred},s'^*}$.

To validate this claim, note that if Equation \ref{eq:eps_nu_beta_condtion} holds, then:

\begin{equation*}
    \begin{split}
        \text{\# of samples }(x,x') \text{ in }\text{pairs}(s_{pred})\text{ such that }f(x,x') = z,\text{ if } \not \exists s^*:\, z \in \{\bar{c}_{s^*_{pred},s^*},\underline{c}_{s^*_{pred},s^*} \} &\leq \\
        \text{\# of samples }(x,x') \text{ in }\mathcal{D}^*(s^*_{pred})\text{ such that }f(x,x') = z,\text{ if } \not \exists s^*:\, z \in \{\bar{c}_{s^*_{pred},s^*},\underline{c}_{s^*_{pred},s^*} \}\,\\
        +\,|\text{pairs}(s_{pred})\setminus\mathcal{D}^*(s^*_{pred})|&\leq \\
 \epsilon(|\tau_A| + |\tau_B|) +
        \,|\text{pairs}(s_{pred})\setminus\mathcal{D}^*(s^*_{pred})|&\leq \\
              \textbf{(Note this line:)}\,\,\,\,\, \epsilon(|\tau_A| + |\tau_B|) +
        \beta(|\tau_A| + |\tau_B|)&\leq \\
        q_{thresh.}(|\tau_A| + |\tau_B|)&<\\
       (\nu(1-\epsilon)/2)(|\tau_A| + |\tau_B|) -
        \beta(|\tau_A| + |\tau_B|) &\leq\\
       (\nu(1-\epsilon)/2)(|\tau_A| + |\tau_B|) -
       |\mathcal{D}^*(s^*_{pred})\setminus\text{pairs}(s_{pred})|&\leq\\
\text{\# of samples }(x,x') \text{ in }\mathcal{D}^*(s^*_{pred})\text{ such that }f(x,x') = z,\text{ if } \exists s^*:\, z \in \{\bar{c}_{s^*_{pred},s^*},\underline{c}_{s^*_{pred},s^*} \} &\\
      -\, |\mathcal{D}^*(s^*_{pred})\setminus\text{pairs}(s_{pred})|&\leq\\
\text{\# of samples }(x,x') \text{ in }\text{pairs}(s_{pred})\text{ such that }f(x,x') = z,\text{ if } \exists s^*:\, z \in \{\bar{c}_{s^*_{pred},s^*},\underline{c}_{s^*_{pred},s^*} \} &
    \end{split}
\end{equation*}
Therefore, for any $s^*_{pred}$,
 we can define:

 \begin{equation}
\mathcal{D}^*_{new\_pairs}(s^*_{pred},j,j') := \{ (x,x') | (x,x') \in \biguplus_{k = j-1}^{j'} \text{pred\_succ}^*[k]\}
 \end{equation}
where
 \begin{equation}
     \text{pred\_succ}^*[k] := \{(x_{h},x_{h+1}) \in \mathcal{D}^*(s^*_{pred},\cdot)|f(x_{h},x_{h+1}) = k\xi - \frac{n_\Xi\xi}{2}\}.
 \end{equation}
If $j$ and $j'$ are chosen as in Line 19 and 20 of Algorithm \ref{alg:main_alg}, then for any pair $(s^*_{pred},s^*)$ there is a unique set $\mathcal{D}^*_{new\_pairs}(s^*_{pred},j,j')$ containing  at least a $(1-\epsilon)$ fraction of the samples in $\mathcal{D}(s^*_{pred},s^*)$. Furthermore, note that by assumption, summing over \textit{all pairs} $(s^*_{pred},s^*)$, the sets $\mathcal{D}^*_{new\_pairs}(s^*_{pred},j,j') $ and $\mathcal{D}_{new\_pairs} $ can differ by at most $\beta(|\tau_A| + |\tau_B|)$ members in total (because all datasets $\mathcal{D}^*(s^*_{pred},\cdot)$ and $\text{pairs}(s_{pred})$ differ by at most this many members in total). Therefore, we have shown that, with high probability, each pair of observations $x,x'$ is included in the correct $\mathcal{D}_{new\_pairs}$ corresponding to $(\phi_{h}^*(x),\phi_{h+1}^*(x'))$, with up to at most $(\beta+\epsilon)(|\tau_A| + |\tau_B| )$ exceptions. (Furthermore, different sets $\mathcal{D}_{new\_pairs}$ are non-overlapping by construction.)  Then we only need to show that, with high probability, Line 26 of Algorithm \ref{alg:main_alg} will only merge two sets $\mathcal{D}_{new}$ if these sets correspond to the same latent state. By Lemma \ref{lemma:binary_classification}, taking a union bound over all pairs of latent-state sequential latent-state pairs, we have, with probability at least $1-\delta/2$ that, if 
\begin{equation}
        \nu \geq 8(\beta + \epsilon)\text{ and } |\mathcal{D}^*(s_{h}^*,s_{h+1}^*)|  \geq \frac{128}{7} \ln(4\cdot(\text{\# of latent state pairs})\cdot|\mathcal{G}_{h+1}|/\delta)
\end{equation}
then the classifiers trained in Line 25 will have loss greater than 1/2 if and only if the two datasets being compared correspond to the same latent state. 
Also note that $(\text{\# of latent state pairs}) \leq 1/\nu$; then, under the assumption that $|\mathcal{G}_{h+1}| \leq |\Phi| (\leq |\Phi|^2)$, we have 
\begin{equation}
    \frac{128}{7} \ln(4\cdot(\text{\# of latent state pairs})\cdot|\mathcal{G}_{h+1}|/\delta) \leq \frac{128}{7} (\ln(4|\Phi^2|/\delta) + \ln(1/\nu)).
\end{equation}
Note that by Equation \ref{eq:xi_eta_1_bound} (and $\xi = \alpha/4$), we have 
\begin{equation}
 \frac{16\ln(1/\eta)}{\eta \alpha^2} \geq 1,    \label{eq:elim_eta_alpha}
\end{equation}

so we can write:
\begin{equation}
    \frac{128}{7} \ln(4\cdot(\text{\# of latent state pairs})\cdot|\mathcal{G}_{h+1}|/\delta) \leq 
\frac{128\cdot256}{7} \frac{(\ln(4|\Phi^2|/\delta) + \ln(1/\nu))\ln^2(1/\eta)}{\eta^2 \alpha^4}.
    \end{equation}
Also noting that $\epsilon \leq 1$, and  $\ln(4|\Phi|^2/\delta) \geq \ln(4) \geq 1$, and $1/\nu \geq \{\ln(1/\nu), 1\} $, we have:
\begin{equation}
    \frac{128}{7} \ln(4\cdot(\text{\# of latent state pairs})\cdot|\mathcal{G}_{h+1}|/\delta) \leq 
\frac{128\cdot256}{7} \frac{(\ln(4|\Phi^2|/\delta) + \ln(4|\Phi^2|/\delta) )\ln^2(1/\eta)}{\nu\epsilon^2\eta^2 \alpha^4}.
    \end{equation}
Then, because $N_s^2\ln(n_{\Xi}+1) > 0$, and $128\cdot256 \cdot 2/7 \leq 12800$, we have
\begin{equation}
    \frac{128}{7} \ln(4\cdot(\text{\# of latent state pairs})\cdot|\mathcal{G}_{h+1}|/\delta) \leq \frac{12800 (\ln(4|\Phi^2|/\delta) + N_s^2 \ln(n_\Xi+1)) \ln^2(1/\eta)}{\nu\epsilon^2\eta^2 \alpha^4}.
    \end{equation}
Therefore, the number of observations of each latent state pair $|\mathcal{D}^*(s_{h}^*,s_{h+1}^*)|$ assumed in Equation \ref{eq:required_samples_recursive} is sufficient to ensure that all datasets $\mathcal{D}_{new}$ will be merged correctly. Then, by union bound, with probability at least $1-\delta$, the samples corresponding to the indices in $D_{A,h+1}(s)$ and  $D_{B,h+1}(s)$ will correspond to the observations of a unique latent state $s^*$, with up to at most $(\beta + \epsilon)(|\tau_A| + |\tau_B|)$ exceptions.
\end{proof}

The following lemma is essentially Theorem \ref{thm:main_thm}, with a minor additional assumption:
\begin{lemma}
    Assume that Algorithm \ref{alg:main_alg} is given datasets $\tau_A$ and $\tau_B$ such that the assumptions given in Equations \ref{eq:noise_free_policy_property}, \ref{eq:pair_coverage},\ref{eq:alpha_seperation}, and \ref{eq:eta_coverage} all hold.  Additionally, assume that the relative coverage lower-bound $\eta$ can be written in the form 
    \begin{equation}
        \eta = \frac{e^{-n_{\Xi}\alpha/8}}{1 + e^{-n_{\Xi}\alpha/8}}
    \end{equation}
    for some non-negative integer $n_{\Xi}$.
    Then, for any given $\delta, \epsilon_0 \geq 0$, if 
$\forall s_{h}^*, s_{h+1}^*, \text{such that } s_{h}^* \text{ can transition to }s_{h+1}^*,$
\begin{equation}
       |\mathcal{D}^*(s_{h}^*,s_{h+1}^*)| \geq \frac{819200 H^2(\ln(8H|\Phi|^2/\delta) + N_s^2 \ln(n_\Xi+1)) \ln^2(1/\eta)}{\nu\eta^2 \alpha^4} \cdot \max\left(\frac{1}{\nu^2}, \frac{1}{\epsilon_0^2\nu'^2}\right), \label{eq:required_samples_n_xi}
\end{equation}
then, with probability at least $1-\delta$, the encoders $\phi'_h$ returned by the algorithm will each have accuracy on at least $1-\epsilon_0$, in the sense that, under some bijective mapping $\sigma_h : \mathcal{S}_{h} \rightarrow \mathcal{S}_{h}^*$, 
\begin{equation}
    \forall s^*\in \mathcal{S}_h^*,\,\,\,\Pr_{x\sim \mathcal{Q}(s^*,P^e_h)} (\phi'_h(x) = \sigma_h^{-1}(\phi^*_h(x))) \geq 1-\epsilon_0. \label{eq:conclusion_n_xi}
\end{equation}
\label{lem:theorem_n_xi}
\end{lemma}
\begin{proof}
Note that the conclusion applies at timestep $h =1$ vacuously: there is only one latent state, and $\phi_1'$ returns a constant value. Further, $D_{A,1}(s_1)$ and $D_{B,1}(s_1)$ contain exactly the sets of trajectories in $\tau_A$ and $\tau_B$ which visit $s^*_1$ at step 1.

For subsequent steps, we apply Lemma \ref{lem:recursive_lemma} recursively, with:
\begin{itemize}
    \item $\epsilon = \min(\frac{\nu}{8H},\frac{\nu'\epsilon_0}{8H})$
    \item $\beta =(h-1)\epsilon$ 
    \item $\delta_{\text{Lemma \ref{lem:recursive_lemma}}} := \delta/(2H)$.
\end{itemize}
Note that the assumptions in Equations \ref{eq:eps_nu_beta_condtion_1} and \ref{eq:q_thresh_condition} are met, because:
\begin{equation}
   \epsilon = \epsilon + \beta -\beta = h\epsilon - \beta < H\epsilon - \beta = H\min(\frac{\nu}{8H},\frac{\nu'\epsilon_0}{8H}) -\beta < \frac{H\nu}{8H} - \beta < \frac\nu 8 -\beta 
\end{equation}
    
and 
\begin{equation}
    \epsilon + \beta = h\min\left(\frac{\nu}{8H},\frac{\nu'\epsilon_0}{8H}\right) \leq \frac{h\nu}{8H} = q_{thresh.}
\end{equation}
and 
\begin{equation}
\begin{split}
       q_{thresh.}  =&\\  q_{thresh.} + \beta + \frac{\epsilon\nu}{2}-\beta - \frac{\epsilon\nu}{2}=&\\
       \frac{h\nu}{8H}  +  (h-1 + \nu/2)\min\left(\frac{\nu}{8H},\frac{\nu'\epsilon_0}{8H}\right) - \beta  - \frac{\epsilon\nu}{2}\leq&\\
       \frac{h\nu}{8H} + \frac{(h-1 + \nu/2)\nu}{8H} -\beta  - \frac{\epsilon\nu}{2}<&\\ \frac{2H\nu}{8H}  - \frac{\epsilon\nu}{2}- \beta  <&\\
       \frac{\nu(1-\epsilon)}{2} - \beta.&
\end{split}
\end{equation}
Also, note that the assumption of Equation \ref{eq:required_samples_recursive} is met (by comparison to Equation \ref{eq:required_samples_n_xi}, with  $\epsilon = \min(\frac{\nu}{8H},\frac{\nu'\epsilon_0}{8H})$ and $\delta_{\text{Lemma \ref{lem:recursive_lemma}}} := \delta/(2H)$.) Finally, the inductive hypothesis, that $D_{A,h}(s)$ and $D_{B,h}(s)$  correspond to observations of some state $s^*$, with at most $\beta(|\tau_A|+|\tau_B|)$ exceptions, can be shown to hold. In particular, at iteration $h \geq 2$, we have that the input dataset has at most $\beta_h(|\tau_A|+|\tau_B|) = (\beta_{h-1} + \epsilon)(|\tau_A|+|\tau_B|)$ errors: we can confirm that $\beta_h = (h-1)\epsilon = (h-2)\epsilon + \epsilon = \beta_{h-1} + \epsilon$ for all $h \geq 2$, with $\beta_1 = 0$ (because there are no errors in $D_{A,1}(s_1)$ and $D_{B,1}(s_1)$).

Therefore, by induction and union bound, we can conclude that, with probability at least $1-\delta/2$, for each $ h \in [H]$ and each $s^* \in \mathcal{S}^*_{h}$, there exists some $s \in \mathcal{S}_{h}$ that represents approximately the same set of observations.  In particular, each index in $[|\tau_A|]$ appears in at most one set $D_{A}(s)$ for some $s$ (and likewise for $[|\tau_B|]$ and  $D_{B}(s)$), and there exists some bijective mapping $\sigma_h: \mathcal{S}_{h} \rightarrow \mathcal{S}_{h}^*$, such that for most indices $j$ in $[|\tau_A|]$ 
    \begin{equation}
        j \in D_{A,h}(\sigma^{-1}_h((\phi^*(\tau_{A})_{h}[j])))  
    \end{equation}
    and for most indices $j$ in $[|\tau_A|]$ 
    \begin{equation}
        j \in D_{B,h}(\sigma^{-1}_h(\phi^*((\tau_{B})_{h}[j]))),
    \end{equation}  
    with at most a combined $(H-1)\min(\frac{\nu}{8H},\frac{\nu'\epsilon_0}{8H})(|\tau_A| + |\tau_B| )$ indices in either dataset for which this does not hold.  Note in particular that fewer than $\frac{\nu'\epsilon_0}{8}(|\tau_A| + |\tau_B| )$ indices will be mis-categorized at any timestep. Then, by application of Lemma \ref{lem:multi_class} with $n \geq \nu'(|\tau_A| + |\tau_B| )$, $m = \frac{\nu'\epsilon_0}{8}(|\tau_A| + |\tau_B| )$, $\delta_{\text{Lemma \ref{lem:multi_class}}} = \delta/(2H)$ and $\epsilon=\epsilon_0$, we have that as long as:
    \begin{equation}
        \nu'(|\tau_A| + |\tau_B| ) \geq \frac{64\cdot \max_i|\mathcal{S}_i|\cdot \ln(4H|\Phi|/\delta)}{7\epsilon_0^2}, \label{eq:class_samples_n_xi}
    \end{equation}
    then, by union bound, with probability at least $1-\delta$, the encoders learned on line 46 of Algorithm \ref{alg:main_alg} will each have accuracy at least $1-\epsilon_0$ as in Equation \ref{eq:conclusion_n_xi}, as desired. All that remains to be shown is that Equation \ref{eq:class_samples_n_xi} holds. Note that this equation can be re-written as:
        \begin{equation}
       |\tau_A| + |\tau_B| \geq \frac{64\cdot \max_h|\mathcal{S}_h|\cdot \ln(4H|\Phi|/\delta)}{7\nu'\epsilon_0^2} .
    \end{equation}
    Then, we have:
\begin{equation}
    \begin{split}
        |\tau_A| + |\tau_B|  \geq &\\
          |\mathcal{D}^*(s_{h}^*,s_{h+1}^*)| \geq &\\
\frac{819200 H^2(\ln(8H|\Phi|^2/\delta) + N_s^2 \ln(n_\Xi+1)) \ln^2(1/\eta)}{\nu\eta^2 \alpha^4} \cdot \max\left(\frac{1}{\nu^2}, \frac{1}{\epsilon_0^2\nu'^2}\right) \geq      &\\     
\frac{819200 H^2(\ln(8H|\Phi|^2/\delta) + N_s^2 \ln(n_\Xi+1)) \ln^2(1/\eta)}{\epsilon_0^2\nu'^2\nu\eta^2 \alpha^4}  \geq      &\\ 
\text{(by Equation \ref{eq:elim_eta_alpha})\,\,\,\,\,} \frac{3200 H^2(\ln(8H|\Phi|^2/\delta) + N_s^2 \ln(n_\Xi+1)) }{\epsilon_0^2\nu'^2\nu} \geq      &\\ 
\text{(log. of integer $\geq 0$)\,\,\,\,\,} \frac{3200 H^2\ln(8H|\Phi|^2/\delta)  }{\epsilon_0^2\nu'^2\nu} \geq      &\\ 
\text{(By definition, ($1/\nu' \geq  \max_h|\mathcal{S}_h|$)\,\,\,\,\,} \frac{3200 H^2 \max_h|\mathcal{S}_h| \ln(8H|\Phi|^2/\delta)  }{\epsilon_0^2\nu'\nu} \geq      &\\ 
   \frac{64\cdot \max_h|\mathcal{S}_h|\cdot \ln(4H|\Phi|/\delta)}{7\nu'\epsilon_0^2} &    \end{split}
\end{equation}
completing the proof.
\end{proof}
Finally, we prove Theorem \ref{thm:main_thm}:

\mainthm*
\begin{proof}
This final theorem follows close-to-directly from Lemma \ref{lem:theorem_n_xi}, with the caveat that we no longer assume that $\eta = (e^{-n_{\Xi}\alpha/8})/(1 + e^{-n_{\Xi}\alpha/8})$
    for some non-negative integer $n_{\Xi}$. To do this, it is important to note that the provided $\eta$ is a \textit{lower bound}: if we replace $\eta$ in the algorithm with any arbitrary $\eta' \leq \eta$, then Lemma \ref{lem:theorem_n_xi} will still apply, with a sample-complexity in terms of $\eta'$ rather than $\eta$. Similarly, $\alpha$ is a lower-bound, and so Lemma \ref{lem:theorem_n_xi} will apply for any smaller $\alpha'$. Our task is then to replace $\eta$ and $\alpha$ with some $\eta'$ and $\alpha'$, such that the \textit{asymptotic} sample complexity as given by Equation \ref{eq:asymptotic_complexity} still applies. For simplicity, we can write the sample-complexity given in Equation \ref{eq:required_samples_n_xi} as:

\begin{equation}
           |\mathcal{D}^*(s_{h}^*,s_{h+1}^*)| \geq \frac{(C_1 + C_2 \ln(n_\Xi+1)) \ln^2(1/\eta')}{\eta'^2 \alpha'^4}, 
\end{equation}
where $C_1$ and $C_2$ are independent of $\alpha$, $\eta$, and $n_{\Xi}$.
Now, we must choose $\eta'$ such that:
    \begin{equation}
        \eta \geq \eta' =  \frac{e^{-n_{\Xi}\alpha/8}}{1 + e^{-n_{\Xi}\alpha/8}} \left( = \frac{1}{1 + e^{n_{\Xi}\alpha/8}}\right).
    \end{equation}
An obvious choice is to take (recalling that by definition, $\eta \leq 1/2$, so $\ln(\eta^{-1}-1) > 0$):
\begin{equation}
    n_{\Xi} := \lceil 8\ln(\eta^{-1}-1)/\alpha\rceil 
\end{equation}
so that:
\begin{equation}
    \eta' = \frac{1}{1 + e^{\lceil 8\ln(\eta^{-1}-1)/\alpha\rceil\alpha/8}}
\end{equation}
Then we have:
\begin{equation}
    \eta \geq \eta' \geq \frac{1}{1 + e^{( 8\ln(\eta^{-1}-1)/\alpha+1)\alpha/8}} \geq \eta \cdot e^{-\alpha/8}
\end{equation}
so that we have sufficient samples for Lemma \ref{lem:theorem_n_xi} if:
\begin{equation}
           |\mathcal{D}^*(s_{h}^*,s_{h+1}^*)| \geq \frac{(C_1 + C_2 \ln(   8\ln(\eta^{-1}-1)/\alpha'+2)) (\ln(1/\eta) + \alpha'/8 )^2  e^{\alpha'/4}}{\eta^2 \alpha'^4},  \label{eq:theor_pre_alpha}
\end{equation}
Strictly speaking, Equation \ref{eq:theor_pre_alpha}  with $\alpha' = \alpha$ satisfies the ``big-O'' asymptotic complexity given in Equation \ref{eq:asymptotic_complexity} in terms of $1/\alpha$ and $1/\eta$ as these quantities approach infinity. However, if we just take $\alpha' = \alpha$, notice that Equation \ref{eq:theor_pre_alpha} seems to require an exponentially large number of samples for large $\alpha$. Recall though that $\alpha$ is a \textit{lower bound}, so we can simply select an arbitrarily lower $\alpha'$ in the case of large $\alpha$. In particular, if we take $\alpha' = \min(1,\alpha)$, then $\alpha' \leq \alpha$ as needed, and (ignoring lower-order polynomial terms and all logarithmic factors), the dependence of our sample complexity on $\alpha$ becomes:
\begin{equation}
\min(e^{\alpha/4}/\alpha^4, e^{1/4}/1^4) = \min(e^{\alpha/4}/\alpha^4, e^{1/4}) \leq C \cdot 1/\alpha^4
\end{equation}
so that the sample complexity is bounded even for large $\alpha$. 

These modifications to $\eta$ and $\alpha$ are performed on Lines 1-3 of Algorithm \ref{alg:main_alg}, so the overall asymptotic sample complexity given in Equation \ref{eq:asymptotic_complexity} holds for the algorithm overall, with the input $\alpha$ and $\eta$.
\end{proof}

\section{Experiment Details} \label{sec:experiment_details}
For the hyperparameters $\eta, \nu$ and $\alpha$ of CRAFT, we use the ``population'' values based on the ground-truth dynamics and policy. In other words, we set:
\begin{equation}
\begin{split}
        e^\alpha= \min_{s_h^*,s_{h+1}^*,s_{h+1}'^*} \max\Bigg[&\left(\frac{\Pr_{\pi_A}(s_{h+1}'^*|s_h^*)/\Pr_{\pi_A}(s_{h+1}^*|s_h^*)}{\Pr_{\pi_B}(s_{h+1}'^*|s_h^*)/\Pr_{\pi_B}(s_{h+1}^*|s_h^*)}\right), \\
        &
        \left(\frac{\Pr_{\pi_B}(s_{h+1}'^*|s_h^*)/\Pr_{\pi_B}(s_{h+1}^*|s_h^*)}{\Pr_{\pi_A}(s_{h+1}'^*|s_h^*)/\Pr_{\pi_A}(s_{h+1}^*|s_h^*)}\right)\Bigg] = \frac{0.75/0.25}{0.5/0.5} = 3
\end{split}
\end{equation}
so $\alpha = \ln(3)$, and 
\begin{equation}
    \nu = \min_{s_h^*,s_{h+1}^*} \frac{\Pr_{\pi_A}(s_h^*,s_{h+1}^*) +\Pr_{\pi_B}(s_h^*,s_{h+1}^*)}{2} = \frac{1/4 + 1/16}{2} = \frac{5}{32},
\end{equation}
and 
\begin{equation}
    \eta = \min_{s_h^*,s_{h+1}^*} \frac{\Pr_{\pi_B}(s_h^*,s_{h+1}^*)}{\Pr_{\pi_A}(s_h^*,s_{h+1}^*) +\Pr_{\pi_B}(s_h^*,s_{h+1}^*)} = \frac{1/16}{1/4 + 1/16} = \frac{1}{5}.
\end{equation}
For the ``Single observation classification'' and ``Paired observation classification'' baselines, we select the feature $\phi_h$ (or feature-pair  $\phi_h,\phi_{h+1}$) such that the mutual information between $\phi_h(x_h)$ (respectively, $(\phi_h(x_h),\phi_{h+1}(x_{h+1}))$) and the agent's identity is maximized on the collected trajectories.

The ``Average Encoder Accuracy'' was computed based on the ``population'' behavior of the environment: that is, the accuracy of encoder $\phi'_h$ which extracts the feature $(s_h^* \text{ XOR } e^i_h)$ from $x_h$ is computed as  $\max(\Pr(e^i_h = 1), \Pr(e^i_h = 0))$, which can be determined analytically from the parameters of Markov chain $e^i$. For a given algorithm, this quantity was then averaged over timesteps for the returned encoder.

For the ``Paired observation classification'' baseline, note that for timesteps $h=2$ through $h=H-1$, the suggested baseline could refer two distinct encoders: the encoder $\phi_h$ such that $\phi_h(x_h)$ and some $\phi_{h+1}(x_{h+1})$ are together most informative at predicting the agent observing $(x_h,x_{h+1})$; \textit{or} the encoder $\phi_h'$ such that $\phi_h'(x_h)$ and some $\phi_{h-1}(x_{h-1})$ are together most informative at predicting the agent observing $(x_{h-1},x_{h})$. In reporting the final encoder accuracies, took the average accuracy of these two encoders.

\end{document}